\documentclass[%
 sor,
 jor,
 amsmath,amssymb,
 reprint,%
 floatfix,%
]{revtex4-2}
 
\usepackage{url}
 
\usepackage[utf8]{inputenc} 
\usepackage[T1]{fontenc}    
\usepackage[colorlinks=true, linkcolor=black, citecolor=black, urlcolor=black]{hyperref} 
\usepackage{booktabs}       
\usepackage{amsfonts}       
\usepackage{nicefrac}       
\usepackage{microtype}      
\usepackage{pifont} 
\usepackage{wrapfig}        
\usepackage[table]{xcolor}  
\usepackage{graphicx}
\usepackage{algorithm}
\usepackage{algpseudocode}
\makeatletter\def\ftype@algorithm{8}\makeatother 
\makeatletter\renewcommand{\theHALG@line}{\thealgorithm.\arabic{ALG@line}}\makeatother 
 
\usepackage{multirow}
 
\usepackage{caption}
\DeclareCaptionJustification{justifiedleft}{\leftskip0pt\rightskip0pt\parfillskip0pt plus 1fil\relax}
\usepackage{amsmath}
\usepackage{amsthm}
\usepackage{amssymb}
\usepackage[capitalise]{cleveref}
 
\usepackage{verbatim}
\usepackage{tikz}
 
\usepackage{threeparttable}
\definecolor{myred}{HTML}{EF476F}
\definecolor{myblue}{HTML}{1F64DC}
\definecolor{mygreen}{HTML}{02c39a}
 
\usepackage{array}
\usepackage{textcomp}

\DeclareMathOperator*{\argmax}{argmax}
\DeclareMathOperator*{\argmin}{argmin}
\newtheorem{theorem}{Theorem}
\newtheorem{proposition}{Proposition}

\newcommand{\bs}{\mathbf{s}}
\newcommand{\bx}{\mathbf{x}}
\newcommand{\bn}{\mathbf{n}}
\newcommand{\smath}[1]{{\small$#1$}}

\makeatletter\renewcommand\paragraph{\@startsection{paragraph}{4}{\z@}{0.5\baselineskip}{-1em}{\normalfont\normalsize\bfseries}}\makeatother
\begin{document}
 
\preprint{AIP/123-QED}
 
\title{Enabling Unsupervised Training of Deep EEG Denoisers With Intelligent Partitioning}
 
\author{Qiyu Rao}
\email[Corresponding author: ]{kianna.rao21@imperial.ac.uk}
\affiliation{Department of Electrical and Electronic Engineering, Imperial College London}
\author{Haozhe Tian}
\affiliation{Dyson School of Design Engineering, Imperial College London}
\author{Homayoun Hamedmoghadam}
\author{Danilo Mandic}
\affiliation{Department of Electrical and Electronic Engineering, Imperial College London}
 
\begin{abstract}
Denoising electroencephalogram (EEG) is an inherently challenging task, since neural activity is not only subtle but also inseparable from spectrally overlapping noise artifacts. Today, effective EEG denoising is more important than ever, given the rapid adoption of wearables across various applications. Deep learning methods have shown promising results in decomposition-free denoising that handles the time-varying pervasive EEG artifacts. However, training highly expressive neural networks requires artifact-free EEG, which is inherently unobtainable. To address this, we propose Intelligent Partitioning for Self-supervised Denoising (iPSD). Our method eliminates the need for clean references by learning to partition an input EEG segment into independent noisy realizations with the same underlying signal. This enables self-supervision of deep learning denoisers, even in zero-shot settings where only a single EEG segment to be denoised is available. We validate iPSD through extensive experiments, including validations on wearable EEG from in-ear sensors. The results show that iPSD achieves state-of-the-art performance, most notably under extremely low signal-to-noise ratios (down to $-10$ dB) and challenging artifacts (e.g., EMG), with spectral fidelity orders of magnitude higher than competitive baselines.
\end{abstract}
 
\maketitle

\section{Introduction}\label{sec:intro}
Wearable technologies are transforming continuous health monitoring, with significant advances in electrocardiogram (ECG)~\citep{tian2025machine} and photoplethysmogram (PPG)~\citep{zhang2023secure} tracking. However, progress in wearable electroencephalogram (EEG) monitoring has been comparatively slow: on top of the poor electrode-skin contact, motion artifacts, and physiological interference inherent to wearable recordings, the brain signals themselves are also subtle, broadband, and non-stationary~\citep{chu2021ahed, huhn2022impact, rao2025panorama}.

A range of signal processing methods has been proposed to denoise EEG through decomposition and selective reconstruction. One line of work decomposes the signal using wavelet bases and reconstructs a denoised version by thresholding wavelet coefficients~\citep{krishnaveni2006automatic, alyasseri2019eeg}. However, this approach struggles when the noise spectrum overlaps with that of the EEG, as is the case for electromyographic (EMG) artifacts from muscle contractions, a prevalent noise source in wearable recordings. Another approach partially alleviates this by decomposing the signal into data-driven mode functions~\citep{huang1998empirical, dragomiretskiy2013variational}. However, they are sensitive to hyperparameter choices and abrupt perturbations common in wearable EEG.

Deep learning (DL) offers a compelling alternative to decomposition-based denoising, as highly expressive neural networks can directly map a noisy EEG to its clean counterpart without any fixed or heuristic basis~\citep{yang2016removal, sun2020novel, jurczak2022implementation, chuang2022ic}. However, supervised training of DL denoisers requires clean reference recordings, which are unattainable since neural signals are inseparable from physiological and environmental noise at the measurement point. Self-supervision, using only independent noisy realizations, has proven effective for denoising images with the hugely successful Noise2Noise (N2N)~\citep{lehtinen2018noise2noise} method. N2N learns to restore the clean image by training a neural network to map one noisy realization of the image to another. \citet{zs-N2N} later extended N2N to settings without multiple measurements of the same image via interleaved downsampling, which assigns alternating pixels to two sub-images to form a pair of noisy realizations. This heuristic, however, does not straightforwardly transfer to EEG: unlike on locally smooth images, EEG's rapid oscillatory structure causes interleaved downsampling to yield realizations with distinct underlying clean signals, thus violating N2N's core assumption.

To restore clean EEG signals without any reference, we propose Intelligent Partitioning for Self-supervised Denoising (iPSD). The core idea is to learn to assign samples of a noisy input into two sub-signals with the same underlying clean signal. (Rather than assignment via a predetermined rule, such as in interleaved downsampling, the intelligent partitioning enables learning signal-specific assignment.) The two sub-signals then act as a pair of noisy realizations for N2N-style self-supervised denoising.
The intuition is that, for two sub-signals sharing the same clean signal but carrying independent noise, the noise in one is unpredictable from the other; so a denoiser network trained to map one to the other is driven to converge to output their shared component, i.e., the clean signal. 
In iPSD, the partitioning is optimized via reinforcement learning~(RL), jointly with the denoiser, using the negative of the converged denoising loss as the reward. We show in \cref{sec:train} that, under reasonable assumptions, this drives the partitioning module toward sub-signals that are maximally informative of one another’s underlying signal. We further develop a zero-shot variant, named iPSD-Zero, that recovers the clean signal from a single noisy segment without any prior training.

We evaluate iPSD on EEG signals corrupted by White Gaussian Noise (WGN) and real-world EMG artifacts---across noise levels up to ten times the signal power---as well as real noisy wearable EEG data from custom-built in-ear sensors. Across all conditions, iPSD consistently outperforms state-of-the-art methods with spectral fidelity orders of magnitude higher than established baselines. On a downstream sleep-stage classification task, wearable EEG denoised by iPSD achieves accuracy comparable to clinical-grade scalp EEG, underscoring the method's practical value for accessible health monitoring. To the best of our knowledge, iPSD is the first self-supervised EEG denoiser to demonstrate such effectiveness across noise types, SNR levels, and real-world wearable recordings.

\section{Methodology}\label{Sec:Methodology}
Let $\bx\in\mathbb{R}^L$ denote a clean EEG signal of length $L$, and let $\bs=\bx+\bn$ be a noise-corrupted measurement of $\bx$, where $\bn \in\mathbb{R}^L$ denotes additive noise that is independent of $\bx$, i.e., $\mathbb{E}\left[\bn^\top \bx\right]=0$. Denoising then refers to recovering the underlying clean signal $\bx$ from its noisy measurement $\bs$. We focus on deep-learning denoisers, which employ expressive neural networks to directly map the noisy $\bs$ to its clean estimate $\hat{\bx}$, bypassing the fixed or heuristic bases of classical signal processing.

\subsection{iPSD formulation}
In most practical scenarios, the clean reference $\bx$ is unavailable, making supervised training of the denoiser infeasible. Unlike images or audio, where clean reference signals can be synthetically generated, EEG recordings are inherently contaminated. This is because the true neural signal $\bx$ is inseparable from physiological and environmental noise at the point of acquisition, no matter how precise the measurement device. The design of iPSD is optimized for this fundamentally constrained setting, where the true target is physically inaccessible rather than merely scarce.

\begin{figure}[t]
    \centering
    \includegraphics[width=1\linewidth]{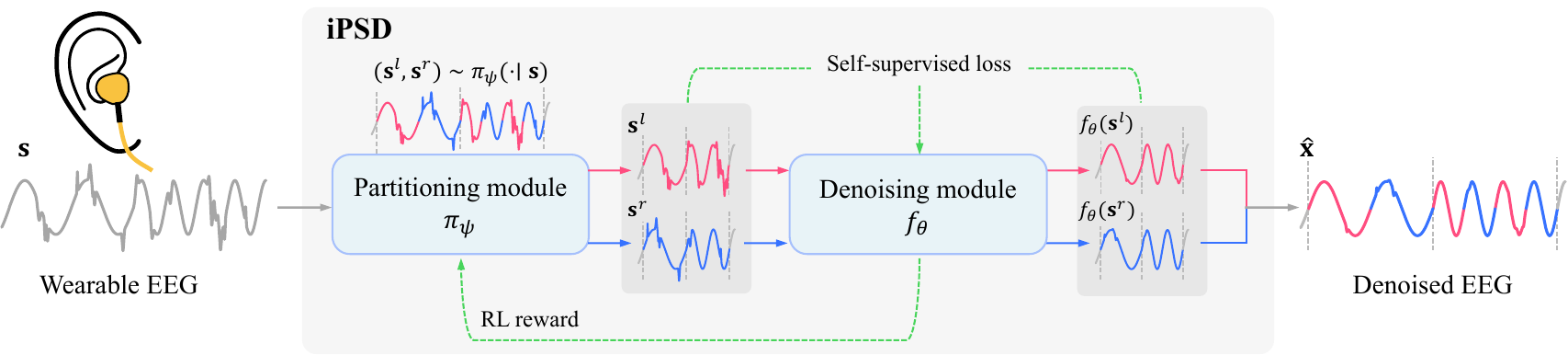}
    \caption{Schematic overview of iPSD. The partitioning module $\pi_{\psi}$ splits the input signal into two sub-signals, $\mathbf{s}^l$ and $\mathbf{s}^r$, which are denoised by the denoising module $f_{\theta}$. Training of $f_{\theta}$ uses the self-supervised N2N loss, $\left\| f_{\theta}(\mathbf{s}^{l})-\mathbf{s}^{r} \right\|_{2}^{2} + \left\| f_{\theta}(\mathbf{s}^{r})-\mathbf{s}^{l} \right\|_{2}^{2}$. The negative of the converged $f_{\theta}$ loss is then used as a reward to update $\pi_{\psi}$. We alternate the updates of $\pi_{\psi}$ and $f_{\theta}$ until convergence.
    }
    \label{fig:metho}
\end{figure}

The schematic in~\cref{fig:metho} depicts the architecture of iPSD, which consists of two learnable modules: the \emph{partitioning module} takes a noisy input $\bs$ and uses a neural network $\pi_\psi$ parametrized by $\psi$ to generate a pair of noisy sub-signals $(\bs^l,\bs^r)$; and the \emph{denoising module} uses a neural network $f_\theta$ parametrized by $\theta$ to map a noisy signal to its clean estimate. The parameters of the two neural networks are jointly optimized via the following self-supervised objective
\begin{equation}
    \theta^\star, \psi^\star = \argmin_{\theta, \psi}\mathbb{E}_{(\bs^l, \bs^r) \sim \pi_\psi(\cdot\mid\bs)}\left[ \left\| f_{\theta}(\bs^{l})-\bs^{r} \right\|_{2}^{2} + \left\| f_{\theta}(\bs^{r})-\bs^{l} \right\|_{2}^{2}\right].
\label{eq:opt}
\end{equation}

\paragraph{Self-supervised denoising}
The objective in~\cref{eq:opt} trains the denoising model $f_\theta$ using only the noisy sub-signals $(\bs^l,\bs^r)$ with no clean reference required. This self-supervision is valid when the two sub-signals share the same underlying clean signal: the independent noise components in $\bs^l$ and $\bs^r$ are each unpredictable from the other, so $f_\theta$ trained to map one sub-signal to the other is driven to output their shared component, i.e., the clean signal. The theorem below, adapted from~\citet{zs-N2N}, formalizes this:
\begin{theorem}
\label{theorem:eq}
Suppose \((\bs^l, \bs^r) \sim \pi_\psi(\cdot \mid \bs)\) are independent noisy realizations of the same underlying signal \(\bx\). Then the \(\theta^\star\) that optimizes the self-supervised loss in \cref{eq:opt} also minimizes the expected $L2$ distance between the network outputs and the ground-truth \(\bx\), i.e.,
\begin{equation*}
\theta^\star = \argmin_{\theta}\, \mathbb{E}
\Big[ \big\| f_{\theta}(\bs^l) - \bx \big\|_{2}^{2} + \big\| f_{\theta}(\bs^r) - \bx \big\|_{2}^{2} \Big].
\end{equation*}
\end{theorem}
The proof of~\cref{theorem:eq} is provided in Appendix~\ref{ap:eq_proof}. This theorem shows that, given the partition $(\bs^l, \bs^r)$, $f_\theta$ can be optimized via gradient descent on the loss
\begin{equation}
    \mathcal{L}(\theta)=\left\| f_{\theta}(\bs^{l})-\bs^{r} \right\|_{2}^{2} + \left\| f_{\theta}(\bs^{r})-\bs^{l} \right\|_{2}^{2}.
\label{eq:n2nloss}
\end{equation}
Crucially, the validity of this self-supervised training scheme hinges on the two sub-signals $(\bs^l, \bs^r)$ sharing the same underlying clean signal. Achieving this is particularly challenging for EEG, whose non-stationarity and fast oscillatory structure demand a flexible, targeted partitioning strategy. This motivates the learnable partitioning module at the core of iPSD.

\paragraph{Intelligent partitioning}
Here, we describe the partitioning module that obtains the pair $(\bs^{l}, \bs^{r})$ from the input noisy signal $\bs$ (see~\cref{fig:partition}). Two ideas are combined here: (i) the input is segmented into small windows where the underlying signal is approximately stationary, and (ii) the model \(\pi_\psi\) learns how to best partition each window to optimize the denoising performance by stochastically exploring all candidate partitions.

Formally, the input $\bs\in\mathbb{R}^L$, with index set \smath{I=\{1,\cdots,L\}}, is segmented into non-overlapping local windows of an appropriate even length $W$ (that divides $L$). The index set of the $k$-th window is $I^{(k)}=\{(k-1)W+1,\cdots,kW\}$, with \smath{k=1,\cdots,L/W}. The partitioning model \(\pi_\psi\) learns to extract from each $I^{(k)}$ (corresponding to window $k$), a subset $I^{(k),l}$, constrained to cardinality $W/2$ (and thus its complement \smath{I^{(k),r} = I^{(k)} \setminus I^{(k),l}}). We use \((\bs^l, \bs^r) \sim \pi_\psi(\cdot \mid \bs)\) to denote the extraction of the signal pair $(\mathbf{s}^{l}, \mathbf{s}^{r})$ according to the stochastic policy $\pi_\psi$: \smath{\mathbb{R}^L \rightarrow \mathcal{P}(I^l)} by  
\begin{equation}
    \bs^l = [\bs_i]_{i \in I^l} \in \mathbb{R}^{L/2}, 
    \qquad 
    \bs^r = [\bs_i]_{i \in I^r} \in \mathbb{R}^{L/2},
\label{eq:partition}
\end{equation} 
where \smath{I^l = \bigcup_{k=1}^{L/W} I^{(k),l}} and \smath{I^r = \bigcup_{k=1}^{L/W} I^{(k),r}}.

\begin{figure}[t]
\centering
\includegraphics[width=1\linewidth]{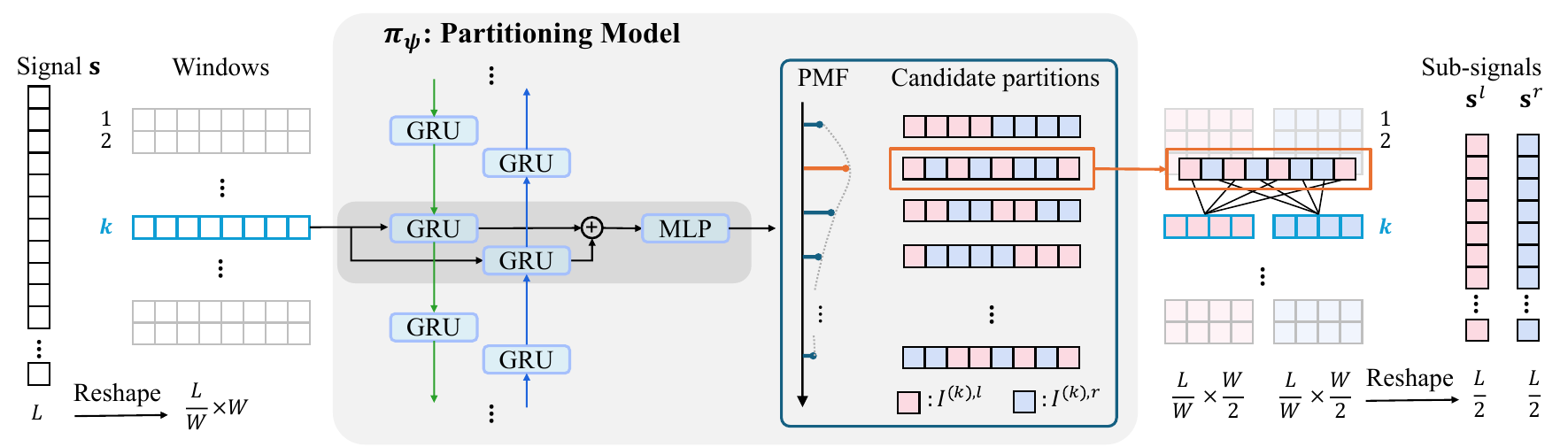}
\caption{Workflow of the partitioning module. The input signal $\bs$ of length $L$ is reshaped into $L/W$ non-overlapping windows of length $W$, each admitting $\tbinom{W}{W/2}/2$ candidate partitions into two equal-length subsets. The policy $\pi_\psi$ selects a partition for each window, and the corresponding subsets are concatenated across windows to form the sub-signal pair. The denoising module $f_\theta$ then processes the sub-signals and provides feedback used to optimize $\pi_\psi$ (see~\cref{fig:metho}).}
\label{fig:partition}
\end{figure}

\begin{proposition}
\label{prop_part}
For any signal window \(I^{(k)}\), the partitioning strategy used in iPSD can always achieve a partition that is at least as good as interleaved downsampling \citep{zs-N2N}, in terms of the mismatch between the underlying clean components of the sub-signals \((\bs_{I^{(k),l}}, \bs_{I^{(k),r}})\).
\end{proposition}
We provide the proof of~\cref{prop_part} in Appendix~\ref{ap:prop_part}. 

\subsection{Training iPSD}\label{sec:train}
The combinatorial nature of the partitioning policy $\pi_\psi$ renders the objective in~\cref{eq:opt} non-differentiable with respect to the parameters $\psi$. To address this, iPSD uses an alternating optimization scheme, where $f_\theta$ is optimized via gradient descent and $\pi_\psi$ via RL.

Let $\mathcal{S}$ denote a set of noisy training signals. In each iPSD training iteration, we sample $\bs\in\mathcal{S}$ and use $\pi_\psi$ to partition it into $(\bs^l, \bs^r)$. We then update $f_\theta$ via gradient descent on the loss in~\cref{eq:n2nloss} until convergence, defined as the variance of loss values in the last $10$ steps falling below $10^{-6}$. The negative of the converged denoising loss serves as the reward signal for the RL-based optimization of~$\pi_\psi$. To reduce reward variance, we apply an averaging scheme. Suppose $\theta$ converges at the $N$th step, and let $J^d_n$ denote the value of $\mathcal{L}(\theta)$ at step $n$. The reward $R(\bs^l, \bs^r)$ is then defined as:
\begin{equation}
    R(\bs^l, \bs^r) \;=\; -\frac{1}{10}\sum_{n=N-9}^{N} J_n^d \, .
\label{eq:return}
\end{equation}
In Appendix~\ref{ap:reward_just}, we show that under reasonable assumptions, this RL reward formulation drives $\pi_\psi$ toward $\bs^l$ and $\bs^r$ that are maximally informative of each other's underlying signal.

\begin{algorithm}[t]
\caption{Training iPSD}
\label{alg:pg}
\begin{algorithmic}[1]
\Statex \textbf{Initialization:} $f_\theta$, $\pi_\psi$, $\mathcal{S}$; learning rates $\eta^\theta$, $\eta^\psi$; $N=0$.
\While{not converged}
  \State Sample $\bs\in\mathcal{S}$ and partition $\bs$ into $(\bs^l, \bs^r)\sim\pi_\psi(\cdot\mid\bs)$
  \While{$\theta$ not converged}
    \State $J^d_n =\left\| f_{\theta}(\bs^{l})-\bs^{r} \right\|_{2}^{2} + \left\| f_{\theta}(\bs^{r})-\bs^{l} \right\|_{2}^{2}$
    \State $\theta\gets\theta - \eta^\theta \nabla_\theta J^d_n$;\; $N=N+1$
  \EndWhile
  \State $\psi\gets\psi + \eta^\psi \nabla_\psi \log\pi_\psi(\bs^l, \bs^r\mid \bs)\, R(\bs^l, \bs^r)$, $\quad R(\bs^l, \bs^r)=-\frac{1}{10}\sum_{n=N-9}^{N}J_n^d$
\EndWhile
\end{algorithmic}
\end{algorithm}


Using the scalar reward $R(\bs^l, \bs^r)$ enables gradient-based optimization of the partitioning policy $\pi_\psi$. To find $\pi_\psi$ that maximizes the reward, we apply the policy gradient theorem (see Appendix~\ref{ap:grad_as}), which shows that the gradient of the expected reward with respect to $\psi$ can be written as
\begin{equation}
    \nabla_\psi\,\mathbb{E}_{(\bs^l,\bs^r)}\bigl[R(\bs^l,\bs^r)\bigr] 
    = \mathbb{E}_{(\bs^l,\bs^r)}\!\left[\nabla_\psi\log\pi_\psi(\bs^l,\bs^r\mid\bs)\,R(\bs^l,\bs^r)\right].
\label{eq:policy_grad}
\end{equation}
We therefore update $\pi_\psi$ via gradient ascent on the right-hand side of \cref{eq:policy_grad}.
\Cref{alg:pg} summarizes the full self-supervised denoising procedure. In practice, we employ Proximal Policy Optimization~(PPO)~\citep{schulman2017proximal} and sample a batch of partitions in parallel to further reduce the variance in the observed $R(\bs^l,\bs^r)$, ensuring stable and efficient policy updates.

\paragraph{Zero-shot setting}
As an extension, we develop iPSD-Zero, a variant of iPSD designed for the challenging \emph{zero-shot} setting, where a clean signal is recovered directly from a single noisy recording without any prior training. This capability is particularly valuable for wearable EEG, where instant deployment is desirable, yet large variabilities in individual physiology, noise characteristics, and hardware make it impractical to form a representative training set before deployment. One significant challenge when only the test signal $\bs$ is available is to rapidly identify the optimal partition from a large search space. We address this by enforcing a shared partitioning strategy across all localized windows, reducing the search space to $\binom{W}{W/2}/2$ candidates. Selecting the optimal partition can be naturally formulated as a multi-armed bandit problem, where each candidate partition corresponds to an \emph{arm}. Pulling each arm yields a reward, defined as in~\cref{eq:return}, which is stochastic due to the inherent randomness of training $f_\theta$. The objective of iPSD-Zero is to identify the best arm, measured by the expected reward, with high probability using as few trials as possible. This is a canonical pure-exploration objective, for which we employ lil'UCB~\citep{jamieson2014lil}, whose sample complexity matches the theoretical lower bound up to constant factors. By iteratively pulling different arms, iPSD-Zero gradually builds empirical estimates of each arm's expected reward and eventually identifies the arm whose corresponding partition of $\bs$ best allows $f_\theta$ to recover the clean signal. Further details are provided in Appendix~\ref{ap:ipsd-zs}.

\section{Experiments}\label{Sec:experiments}
We begin with some notes on the implementation of iPSD. For the partitioning module, we use a window length of 8 samples, chosen by ablation (see details in Appendix~\ref{ap:net_architectures}) to achieve sufficient partitioning flexibility while keeping the search space tractable. The policy network $\pi_\psi$ uses a bidirectional GRU architecture, which outperforms alternatives including MLP, U-Net, and RNN (see Appendix~\ref{ap:net_architectures} for detailed comparisons). We train $\pi_\psi$ via PPO using the Adam optimizer with a learning rate of $10^{-4}$, a batch size of 64, and a total of $20{,}000$ update steps; other hyperparameters follow~\citet{huang2022cleanrl}. We observe stable convergence of iPSD across independent runs with different random seeds (see results in Appendix~\ref{ap:random_seed}).

For the denoising module $f_\theta$, we use the loss function $\mathcal{L}(\theta)$ in~\cref{eq:n2nloss} together with two regularization terms proposed by~\citet{zs-N2N}. Given a partition $(\bs^l, \bs^r) \sim \pi_\psi(\cdot\mid \bs)$, we define functions $\pi^l, \pi^r: \mathbb{R}^L\rightarrow\mathbb{R}^{L/2}$ such that $\bs^l=\pi^l(\bs)$ and $\bs^r=\pi^r(\bs)$. The loss function is
\begin{equation}
\begin{split}
    \mathcal{L}_{\text{full}}(\theta) = \mathcal{L}(\theta) + \| f_{\theta}(\pi^l(\bs))-\pi^l(f_{\theta}(\bs)) \|^{2}_{2} + \| f_{\theta}(\pi^r(\bs))-\pi^r(f_{\theta}(\bs)) \|^{2}_{2},
\end{split}
\end{equation}
where the two regularization terms encourage iPSD's output to obey a consistency property that the underlying clean signal should satisfy: partitioning and then denoising should yield the same result as denoising and then partitioning. We implement $f_\theta$ as a lightweight fully convolutional network with three layers and LeakyReLU activations, totaling approximately $7$k parameters. The channel dimensions progress as $1\rightarrow48\rightarrow48\rightarrow1$, with all layers using a kernel size of 3 and padding of 1. The compact architecture of $f_\theta$ leads to fast convergence (in around 200 epochs) and permits parallel training of multiple instances, allowing the partitioning policy $\pi_\psi$ to quickly explore diverse strategies and identify high-quality partitions.

\paragraph{Experiments on synthetic data}
We benchmark iPSD against state-of-the-art baselines including wavelet-based~\citep{alyasseri2019eeg, houamed2020ecg}, mode decomposition-based~\citep{dora2020correlation, colominas2014improved, ranjan2022motion}, hybrid~\citep{kerechanin2022eeg}, and self-supervised DL~\citep{chen2025self} approaches (see a summary in Appendix~\ref{ap:baselines}). For quantitative evaluation, we use three complementary metrics standard in signal denoising: signal-to-noise ratio (SNR), peak signal-to-noise ratio (PSNR), and the mean squared error of the power spectrum (Spectral MSE). SNR provides a global measure of time-domain fidelity, PSNR emphasizes robustness to large transient spikes, and Spectral MSE assesses spectral fidelity, which is critical because EEG analysis depends heavily on spectral features such as relative band power, peak frequency, and spectral entropy. Together, these metrics provide a balanced evaluation covering both the temporal and spectral domains. The definitions of all metrics are provided in Appendix~\ref{ap:metrics}. Note that the clean signal is solely used for metric computation and is not accessible to any denoising method.

\begin{table*}[t]
\caption{Comparison of denoising performance.}
\centering
\footnotesize
\renewcommand{\arraystretch}{1.1}
\newcommand{\ex}[1]{{\scriptscriptstyle\times 10^{#1}}}
\setlength{\tabcolsep}{1.85pt}
\begin{tabular}{l *{6}{c} | *{6}{c}}
\toprule
& \multicolumn{6}{c|}{\textbf{WGN}}
& \multicolumn{6}{c}{\textbf{EMG}} \\
\cmidrule(lr){2-7}\cmidrule(lr){8-13}
& \multicolumn{3}{c}{$-5$ dB Input SNR}
& \multicolumn{3}{c|}{$0$ dB Input SNR}
& \multicolumn{3}{c}{$-5$ dB Input SNR}
& \multicolumn{3}{c}{$0$ dB Input SNR} \\
\cmidrule(lr){2-4}\cmidrule(lr){5-7}\cmidrule(lr){8-10}\cmidrule(lr){11-13}
{\footnotesize Method}
  & SNR & PSNR & S-MSE
  & SNR & PSNR & S-MSE
  & SNR & PSNR & S-MSE
  & SNR & PSNR & S-MSE \\
\midrule
mVMD
  & $-5.05$ & $9.05$  & $1.46\ex{4}$
  & $2.18$  & $16.28$ & $1.13\ex{4}$
  & $-6.05$ & $8.05$  & $5.40\ex{4}$
  & $0.93$  & $15.03$ & $3.06\ex{4}$ \\
CEEMDAN
  & $-3.85$ & $10.25$ & $0.98\ex{4}$
  & $4.13$  & $18.23$ & $0.78\ex{4}$
  & $0.25$  & $14.35$ & $0.36\ex{3}$
  & $3.26$  & $17.36$ & $0.22\ex{3}$ \\
FrWT
  & $-3.54$ & $10.56$ & $0.96\ex{4}$
  & $3.43$  & $17.53$ & $0.83\ex{4}$
  & $-0.06$ & $14.04$ & $0.55\ex{3}$
  & $3.53$  & $17.63$ & $0.62\ex{3}$ \\
EMD-LoG
  & $2.28$ & $16.38$  & $0.80\ex{4}$
  & $4.40$  & $18.50$ & $0.80\ex{4}$
  & $0.32$  & $14.42$ & $0.38\ex{3}$
  & $3.55$  & $17.65$ & $0.22\ex{3}$ \\
WPT-ICA
  & $-4.89$ & $9.21$  & $1.63\ex{4}$
  & $1.74$  & $15.84$ & $1.23\ex{4}$
  & $-0.09$ & $14.01$ & $0.66\ex{3}$
  & $2.36$  & $16.46$ & $0.31\ex{3}$ \\
\citeauthor{chen2025self}
  & $2.01$ & $16.11$ & $0.28\ex{4}$
  & $3.25$ & $17.35$ & $0.21\ex{4}$
  & $-0.05$ & $14.05$ & $0.42\ex{3}$
  & $3.01$ & $17.12$ & $0.19\ex{3}$ \\
Optimal-WT
  & $3.34$  & $17.44$ & $0.36\ex{4}$
  & $5.07$  & $19.18$ & $0.27\ex{4}$
  & $0.33$ & $14.42$ & $0.69\ex{3}$
  & $3.89$  & $18.10$ & $0.43\ex{3}$ \\
\midrule
iPSD
  & \cellcolor{myred!20}$3.97$  & \cellcolor{myred!20}$18.07$ & \cellcolor{myred!20}$55.14$
  & \cellcolor{myred!20}$5.78$  & \cellcolor{myred!20}$19.90$ & \cellcolor{myred!20}$41.13$
  & \cellcolor{myred!20}$3.63$  & \cellcolor{myred!20}$17.73$ & \cellcolor{myred!20}$53.04$
  & \cellcolor{myred!20}$5.95$  & \cellcolor{myred!20}$20.05$ & \cellcolor{myred!20}$39.21$ \\
iPSD-Zero
  & \cellcolor{myblue!20}$3.89$ & \cellcolor{myblue!20}$18.00$ & \cellcolor{myblue!20}$58.34$
  & \cellcolor{myblue!20}$5.76$ & \cellcolor{myblue!20}$19.86$ & \cellcolor{myblue!20}$48.67$
  & \cellcolor{myblue!20}$3.27$ & \cellcolor{myblue!20}$17.37$ & \cellcolor{myblue!20}$56.62$
  & \cellcolor{myblue!20}$5.31$ & \cellcolor{myblue!20}$19.41$ & \cellcolor{myblue!20}$44.12$ \\
\bottomrule
\multicolumn{13}{l}{%
  * \colorbox{myred!20}{\raisebox{0pt}[0.5em][0em]{Pink}} (\colorbox{myblue!20}{\raisebox{0pt}[0.5em][0em]{blue}}) highlights the best (second-best) performance;\, S-MSE stands for spectral MSE.}
\end{tabular}
\label{tab:mean}
\end{table*}

We use EEG signals from the CHB-MIT Scalp EEG Database~\citep{PhysioNet-chbmit-1.0.0, shoeb2009application}, which was collected at the Children’s Hospital Boston and contains over 900 h long-term scalp EEG recordings from 22 pediatric patients (5 males aged 3--22 years; 17 females aged 1.5--19) with intractable seizures. All signals are sampled at $256$ Hz with 16-bit resolution, using the international 10–20 electrode placement system with a bipolar montage. To generate synthetic data, we first split all recordings into 10-second segments, which are long enough for iPSD and the baselines to capture meaningful neural rhythms~\citep{rosiani2021impact}. To create contaminated EEG signals, we introduce WGN and real-world EMG artifacts. The WGN is generated by sampling from a Gaussian distribution, while the EMG artifacts are from \citet{rantanen2016survey}, recorded on 15 healthy participants (8 female, 7 male; aged 26--57) performing facial movements such as chewing, smiling, lip puckering, and frowning. Since the EMG artifacts were recorded at a different sampling rate, we downsample them to $256$ Hz before adding them to the EEG. The synthetic dataset was divided into training and test sets with an 80/20 split.

\begin{wrapfigure}[16]{r}{0.5\textwidth}
\centering
\vspace{-1em}
\includegraphics[width=0.5\textwidth]{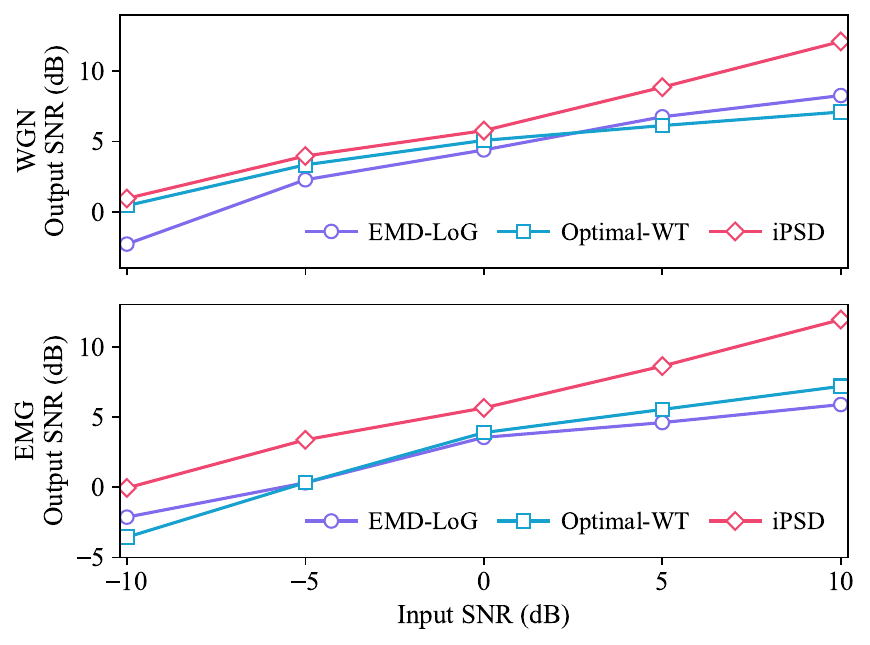}
\caption{Performance over different noise levels.}
\label{fig:all-snr}
\end{wrapfigure}

\Cref{tab:mean} compares the denoising performance of iPSD and its zero-shot variant, iPSD-Zero, against the baselines at input SNR levels of $-5$ dB and $0$ dB. (The zero-shot iPSD is optimized directly on each test signal.) The proposed iPSD consistently and significantly outperforms all baselines across all evaluation metrics, under both WGN and EMG contamination. Compared to the strongest baseline, Optimal-WT, iPSD consistently achieves better denoising performance, with up to 3.3 dB higher output SNR, corresponding to more than $53\%$ lower residual noise power after denoising. We also observe that iPSD is more robust under low-SNR conditions (at $-5$ dB, signal power is only approximately $32\%$ of noise power), where the baselines show limited denoising capability, with several outputting negative SNR values. Most notably, while the spectral MSEs of the baselines are on the order of thousands under WGN and hundreds under EMG noise, the spectral MSE of iPSD remains consistently below 60. To provide an intuitive view of the strong spectral fidelity of iPSD, we additionally present visual comparisons of iPSD against Optimal-WT in Appendix~\ref{ap:syn_qual}.

\Cref{fig:all-snr} compares the output SNR of iPSD against two representative baselines across input SNR levels ranging from the critically low $-10$ dB to $10$ dB. The baselines are Optimal-WT from the wavelet-based family and EMD-LoG from the mode decomposition-based family. Across all noise levels, iPSD outperforms both baselines. We also observe that the baselines' output SNR gradually plateau as input SNR increased, whereas iPSD maintains a steady increase throughout the range.

These results highlight the benefit of the learning-based iPSD over fixed or heuristic partitioning strategies. Performance of the baselines varies across noise types: mode decomposition-based methods (CEEMDAN and EMD-LoG) are less effective for WGN, due to the highly unstructured temporal form, whereas wavelet-based methods (Optimal-WT and FrWT) are less effective for EMG, due to the substantial spectral overlap with EEG activity. The iPSD method, in contrast, maintains robust performance under these challenging conditions.
We also note that iPSD-Zero achieves competitive performance, outperforming all baseline methods and ranking only second to the original iPSD. This slight performance degradation comes with a huge saving in computational cost. Starting from scratch, iPSD-Zero denoises a 10-second EEG signal in 5 seconds, making it suitable for real-time and resource-constrained scenarios where only the incoming signal is accessible. Additionally, while the independence assumption allows for theoretical proof of iPSD's convergence, yet in practice, iPSD can recover the clean signal well without the assumption holding strictly. In fact, iPSD's gain over baselines is most significant on the temporally correlated, real-world EMG noise, which violates the assumption. In Appendix~\ref{ap:assump}, we further quantify iPSD's robustness under different levels of assumption violation using autoregressive (AR) noise with different correlation levels, where iPSD retains competitive performance under markedly correlated noise.

\paragraph{Experiments on real-world wearable data} The proposed iPSD method is also evaluated on real-world EEG recordings collected using our custom-built wearable sensors. The dataset comprised 684 hours of recordings from 37 participants (aged 65--83 years; 17 females, 20 males), sampled at $256$ Hz with 16-bit resolution and segmented into 10-second windows. \Cref{fig:real_data_combined}a illustrates the in-ear sensor used for data collection. To maintain stable skin contact and minimize motion artifacts, the in-ear sensor uses a viscoelastic foam earbud with a stretchable, low-impedance cloth electrode on its outer surface. We also applied conductive gel to the electrode before inserting the sensor into the ear canal to further improve the skin-electrode contact.

\begin{figure}[t]
  \centering
    \includegraphics[width=0.9\linewidth]{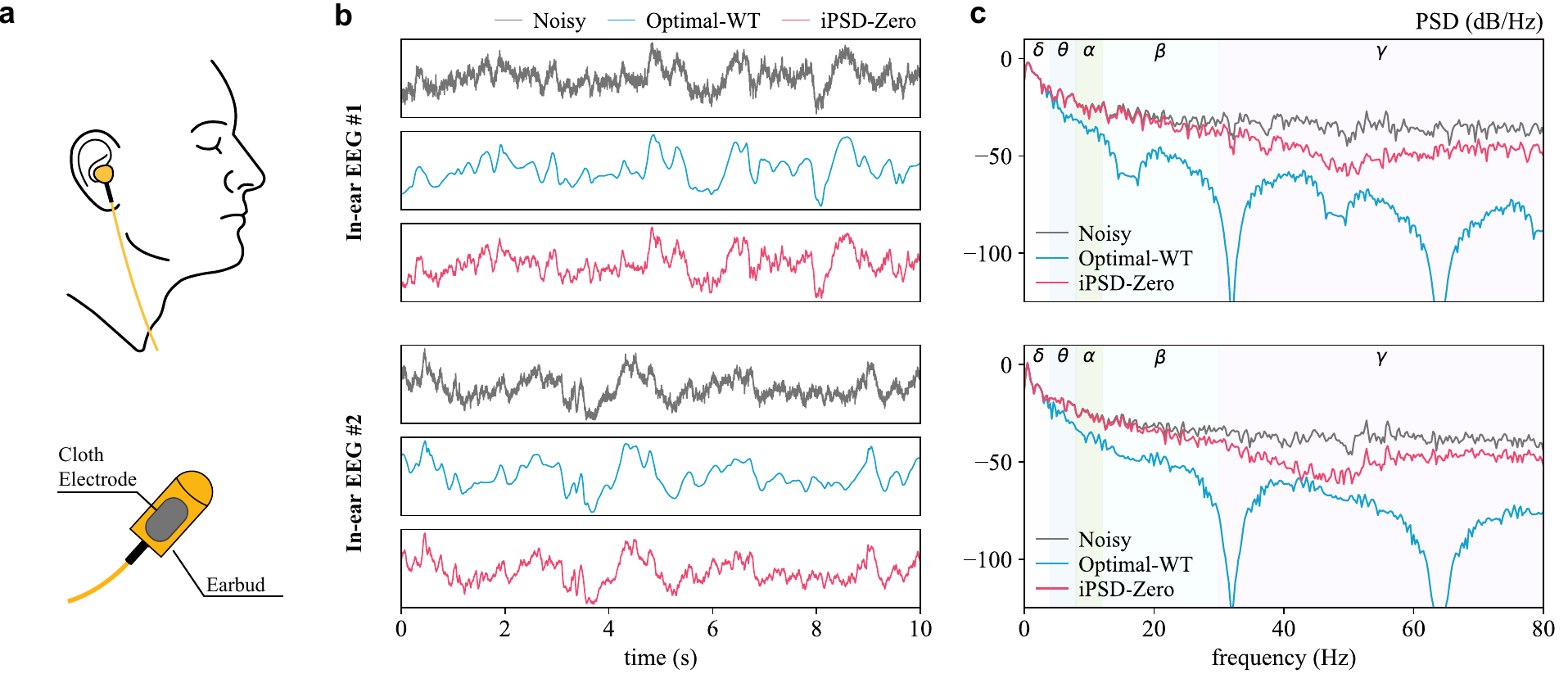}
  \caption{Real-world in-ear EEG denoising with iPSD. (\textbf{a}) Wearable EEG acquisition setup using the in-ear sensor, which integrates a viscoelastic earbud with a stretchable soft-cloth electrode. (\textbf{b}) Noisy in-ear EEG (gray) and the denoised output from Optimal-WT (blue) and iPSD-Zero (pink). (\textbf{c}) PSDs of the noisy in-ear EEG (gray), the Optimal-WT output (blue), and the iPSD-Zero output (pink). The key frequency bands for EEG analysis: $\delta$ ($0.5$--$4$ Hz), $\theta$ ($4$--$8$ Hz), $\alpha$ ($8$--$12$ Hz), $\beta$ ($12$--$30$ Hz), and $\gamma$ ($30$--$80$ Hz) are color-coded in the background.}
  \label{fig:real_data_combined}
  \vspace{-1em}
\end{figure}

Figures \ref{fig:real_data_combined}b and \ref{fig:real_data_combined}c show noisy in-ear EEG recordings, together with the denoising outputs of Optimal-WT and iPSD-Zero. While both iPSD and iPSD-Zero are applicable here, we only show iPSD-Zero in the figure: it operates under the more constrained zero-shot setting, yet produces outputs that are visually indistinguishable from those of iPSD.
\Cref{fig:real_data_combined}b shows time-domain waveforms. The Optimal-WT outputs, although visually smooth, substantially corrupt the original EEG signal. For example, the spindle-like waveform between 7--8.5 s in the lower panel of \cref{fig:real_data_combined}b (In-ear EEG \#2), characteristic of non-REM sleep, is corrupted by Optimal-WT but preserved by iPSD-Zero. The corresponding power spectral densities (PSDs) in ~\cref{fig:real_data_combined}c further support that Optimal-WT is oversmoothing: it exhibits spurious dips at multiple frequencies, indicating that genuine spectral content is being corrupted. In contrast, iPSD-Zero effectively suppresses the spiky, fluctuating noise artifacts while preserving the underlying EEG. Since the in-ear EEG was recorded during sleep, its spectral power is expected to concentrate below $30$ Hz, with the elevated high-frequency power reflecting noise contamination. As shown, iPSD-Zero suppresses the high-frequency noise while preserving the low-frequency EEG components. This visual comparison is consistent with the quantitative results in \cref{tab:mean}, where iPSD substantially outperforms the baselines in spectral fidelity.

\begin{figure}[t]
  \centering
    \includegraphics[width=0.95\linewidth]{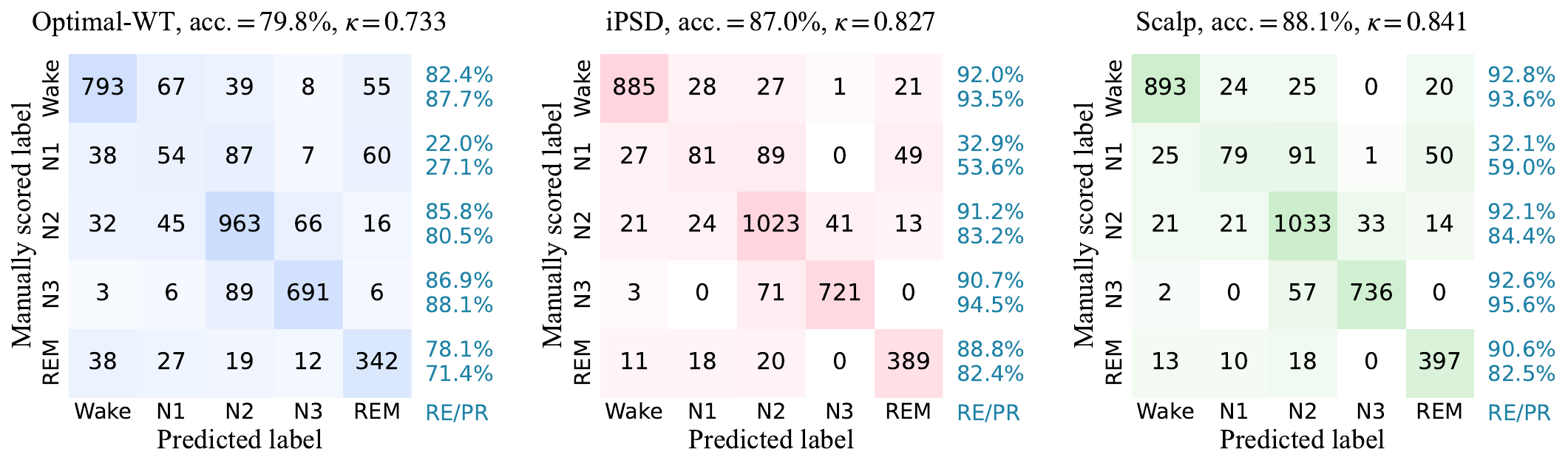}
    \caption{Confusion matrices for sleep-stage classification using the output of the baseline Optimal-WT and the proposed iPSD. Classification results using concurrently recorded high-quality scalp EEG are also included as a reference performance target. The recall and precision (RE/PR) for each sleep stage are shown on the right of each row.}
  \label{fig:confusion_matrix}
  \vspace{-1em}
\end{figure}

Because the real-world in-ear recordings lack clean references, we perform quantitative evaluation by performing sleep stage classification, a cornerstone task in automatic sleep monitoring. The classification is performed over 3,563 in-ear EEG segments, each annotated by experts into one of five sleep stages: Wake, N1, N2, N3, and REM. We denoise the segments using iPSD and the baselines in Appendix~\ref{ap:baselines}. For each method, an XGBoost classifier is trained using a standard set of features~\citep{vallat2021open} extracted from the denoising output. \Cref{fig:confusion_matrix} shows the confusion matrices of (i) Optimal-WT, the best-performing baseline, (ii) iPSD, and (iii) classification based on features extracted from concurrently recorded scalp EEG (C3-M2). The scalp EEG, recorded with the clinical-grade SomnoHD system, is of high quality and serves as a performance reference that the denoising methods aim to match. We observe that classification based on iPSD substantially outperforms that based on Optimal-WT ~(accuracy $87.0\%$, Cohen's $\kappa = 0.827$ vs.\ $79.8\%$, $\kappa = 0.733$), with the largest gains in sleep stages characterized by higher-frequency activity (Wake, N1, and REM). This pattern aligns with \cref{fig:real_data_combined}c, where iPSD shows superior spectral fidelity above $4$ Hz. N1 is consistently the most difficult stage to classify, but iPSD's recall and precision substantially exceed those of Optimal-WT and match the scalp EEG, indicating that the difficulty is inherent to N1 classification~\citep{lee2022interrater} with EEG alone rather than a limitation of denoising. Overall, the per-class metrics of iPSD-denoised in-ear EEG fall within a few points of scalp EEG across all stages, demonstrating that iPSD recovers clinically meaningful information from the noisy wearable recordings at a level comparable to clinical-grade recordings.

\begin{table*}[t]
\caption{Ablating iPSD partitioning by replacing $\pi_\psi$ with ID and DR.}
\centering
\renewcommand{\arraystretch}{1.1}
\footnotesize
\setlength{\tabcolsep}{2.9pt}
\begin{tabular}{l *{6}{c} | *{6}{c}}
\toprule
& \multicolumn{6}{c|}{\textbf{WGN}}
& \multicolumn{6}{c}{\textbf{EMG}} \\
\cmidrule(lr){2-7}\cmidrule(lr){8-13}
& \multicolumn{3}{c}{$-5$ dB Input SNR}
& \multicolumn{3}{c|}{$0$ dB Input SNR}
& \multicolumn{3}{c}{$-5$ dB Input SNR}
& \multicolumn{3}{c}{$0$ dB Input SNR} \\
\cmidrule(lr){2-4}\cmidrule(lr){5-7}\cmidrule(lr){8-10}\cmidrule(lr){11-13}
{\footnotesize Method}
  & SNR & PSNR & S-MSE
  & SNR & PSNR & S-MSE
  & SNR & PSNR & S-MSE
  & SNR & PSNR & S-MSE \\
\midrule
ID
  & $2.58$ & $16.68$ & $62.33$
  & $5.13$ & $19.24$ & $67.18$
  & $1.62$ & $15.73$ & $67.68$
  & $4.31$ & $18.41$ & $60.52$ \\
DR
  & $2.74$ & $16.84$ & $59.43$
  & $5.30$ & $19.41$ & $53.62$
  & $2.20$ & $16.31$ & $63.88$
  & $4.73$ & $18.83$ & $52.27$ \\
iPSD
  & \cellcolor{myred!20}$3.97$  & \cellcolor{myred!20}$18.07$ & \cellcolor{myred!20}$55.14$
  & \cellcolor{myred!20}$5.78$  & \cellcolor{myred!20}$19.90$ & \cellcolor{myred!20}$41.13$
  & \cellcolor{myred!20}$3.63$  & \cellcolor{myred!20}$17.73$ & \cellcolor{myred!20}$53.04$
  & \cellcolor{myred!20}$5.95$  & \cellcolor{myred!20}$20.05$ & \cellcolor{myred!20}$39.21$ \\
iPSD-Zero
  & \cellcolor{myblue!20}$3.89$ & \cellcolor{myblue!20}$18.00$ & \cellcolor{myblue!20}$58.34$
  & \cellcolor{myblue!20}$5.76$ & \cellcolor{myblue!20}$19.86$ & \cellcolor{myblue!20}$48.67$
  & \cellcolor{myblue!20}$3.27$ & \cellcolor{myblue!20}$17.37$ & \cellcolor{myblue!20}$56.62$
  & \cellcolor{myblue!20}$5.31$ & \cellcolor{myblue!20}$19.41$ & \cellcolor{myblue!20}$44.12$ \\
\bottomrule
\multicolumn{13}{l}{%
  * \colorbox{myred!20}{\raisebox{0pt}[0.5em][0em]{Pink}} (\colorbox{myblue!20}{\raisebox{0pt}[0.5em][0em]{blue}}) highlights the best (second-best) performance;\, S-MSE stands for spectral MSE.}
\end{tabular}
\label{tab:abla_partition}
\end{table*}

\paragraph{Ablation study on intelligent partition} We ablate the effect of the proposed intelligent partitioning module by replacing $\pi_{\psi}$ with interleaved downsampling (ID) or differentiable relaxation (DR). ID assigns alternating samples to $(\bs^l, \bs^r)$ and was proven highly effective in self-supervised image denoising \citep{zs-N2N}. DR uses the standard straight-through Gumbel Softmax, which we detail in Appendix~\ref{ap:abla_part}. Experiments are conducted on the same synthetic EEG signals with WGN and EMG contamination. \Cref{tab:abla_partition} shows the denoising performance, where the intelligent partitioning module (in both iPSD and iPSD-Zero) outperforms ID and DR. The fixed ID partitioning performs the worst, as its local-smoothness assumption does not transfer well to rapidly oscillating EEG signals, a limitation that is particularly pronounced for wearable EEG recordings with lower sampling rates~(see Appendix~\ref{ap:lower_fs}). DR is only about $3.5\%$ faster than the RL-based partitioning in iPSD, but at a significant cost to performance (up to $1.33$~dB lower SNR and $33\%$ higher S-MSE). 

A major motivation for the proposed intelligent partitioning module is its flexibility to produce sub-signals with high underlying signal consistency and noise independence, assumptions that underlie most, if not all, reference-free denoising methods~\citep{lehtinen2018noise2noise,huang2021neighbor2neighbor,zs-N2N}. We explicitly verify whether iPSD mitigates violations of these assumptions by measuring both properties in the sub-signals generated by ID, DR, iPSD, and iPSD-Zero (see \cref{tab:underlying} in Appendix~\ref{ap:abla_part}). Compared to ID and DR, iPSD partitioning consistently produces sub-signals that better satisfy the core assumptions of \cref{theorem:eq}, reducing underlying signal mismatch by $42$--$54\%$, noise correlation under WGN by $34$--$43\%$, and noise correlation under real-world EMG artifacts by $78$--$82\%$. These results indicate that iPSD is more robust to the assumption violations that arise in real-world EEG recordings, and that the proposed intelligent partitioning is not merely an improvement to N2N but a key enabler of self-supervised denoising for EEG.

\section{Related work}\label{sec:related_work}
Conventional scalp EEG processing widely uses band-pass and notch filters, which have limited use for low-SNR wearable EEG due to its substantial spectral overlap with prominent artifacts (EMG and EOG)~\citep{uriguen2015eeg}. To handle the localized and nonstationary nature of these artifacts, wavelet-based methods decompose the EEG signal to attenuate or discard components that correlate weakly with a predefined set of wavelet bases before reconstructing the denoised signal~\citep{krishnaveni2006automatic}. However, these methods are sensitive to the choice of wavelet basis, thresholding strategy, and decomposition depth~\citep{grobbelaar2022survey}. Mode decomposition methods overcome the restriction of fixed wavelet bases by decomposing signals into data-driven mode functions~\citep{huang1998empirical, dragomiretskiy2013variational}, but they rely either on perturbation-sensitive local extrema and signal envelopes, or on hyperparameter-sensitive non-convex optimization. Despite the improvements by hybrid methods such as VMD-wavelet and WPT-ICA~\citep{kaur2021eeg, kerechanin2022eeg}, these approaches remain heavily reliant on handcrafted bases and heuristic decomposition strategies that do not generalize well.

Recent advances in DL have opened a promising new avenue for EEG denoising. Highly expressive neural architectures allow direct mapping of noisy EEG signals to their clean counterparts, eliminating the need for handcrafted decomposition strategies \citep{zhang2021novel, gao2022eeg, pu2022eegdnet, chen2023denosieformer}. However, these approaches require a large collection of paired clean-noisy EEG data for training, which is unobtainable in practice. Noise2Noise~(N2N)~\citep{lehtinen2018noise2noise, kashyap21_interspeech} offers a compelling self-supervised image denoising approach based on learning to map independent noisy realizations of a signal to one another. Yet N2N requires paired noisy measurements of the same signal, which is virtually impossible to obtain for EEG signals, as they are unique, unrepeatable observations.
Subsequent works~\citep{krull2019noise2void, huang2021neighbor2neighbor, zs-N2N, chen2025self} seek to learn from a single noisy observation by relying on rule-based partitioning or masking strategies, which are not suited to non-stationary, fast-oscillating EEG signals. In contrast, iPSD learns to generate paired noisy realizations sharing the same underlying clean signal via flexible partitioning of a single measurement.

\section{Conclusion}\label{sec:conclu}
The proposed iPSD (and its zero-shot variant, iPSD-Zero) enables training of highly expressive deep learning EEG denoisers without clean ground-truth signals. Across extensive experiments on synthetic and real EEG data, iPSD and iPSD-Zero consistently outperformed state-of-the-art baselines with markedly higher spectral fidelity. The value of denoising is in empowering tasks that depend on high-quality signals. We demonstrated this in a sleep-stage classification task, where the application of iPSD enhanced classification accuracy up to the standards of scalp EEG. These results suggest that the methods we developed here can effectively narrow the gap between accessible wearable measurements and clinical-grade recordings. Moreover, iPSD-Zero's ability to denoise a single signal without prior training makes it deployable in real-time and on-device, bringing continuous, accessible neural monitoring and diagnosis closer to reality.

\section*{Ethics statement}

This work analyzes data collected from human participants. The study was conducted in accordance with the Declaration of Helsinki and Good Clinical Practice, and received ethical approval from the University of Surrey Ethics Committee (UEC-2019-065-FHMS). All participants received a detailed explanation of the study procedures and provided written informed consent prior to enrollment.

\section*{Acknowledgment}
HT acknowledges financial support from IBM through an IBM PhD Fellowship. Authors acknowledge financial support from Imperial College London through an Imperial College Research Fellowship grant awarded to HH.

\bibliography{ref}

\newpage
\appendix
\setcounter{theorem}{0}
\setcounter{proposition}{0}


\section{Proof of~\cref{theorem:eq}}\label{ap:eq_proof}

\begin{theorem}
Suppose \((\bs^l, \bs^r) \sim \pi_\psi(\cdot \mid \bs)\) are independent noisy realizations of the same underlying signal \(\bx\). Then the \(\theta^\star\) that optimizes the self-supervised loss in \cref{eq:opt} also minimizes the expected $L2$ distance between the network outputs and the ground-truth \(\bx\), i.e.,
\begin{equation*}
\theta^\star = \argmin_{\theta}\, \mathbb{E}
\Big[ \big\| f_{\theta}(\bs^l) - \bx \big\|_{2}^{2} + \big\| f_{\theta}(\bs^r) - \bx \big\|_{2}^{2} \Big].
\end{equation*}
\end{theorem}

\begin{proof}
The independent noisy realization pairs \((\bs^l, \bs^r)\) can be written as $\bs^l = \bx + \bn^l$ and $\bs^r = \bx + \bn^r$, where $\bn^l$ and $\bn^r$ are zero-mean and mutually independent, and each is independent of $\bx$. Expanding the first term of the self-supervised loss in~\cref{eq:opt} gives
\begin{align*}
\mathbb{E}\!\left[\|f_\theta(\bs^l) - \bs^r\|_2^2\right]
&= \mathbb{E}\!\left[\|f_\theta(\bs^l) - \bx - \bn^r\|_2^2\right] \\
&= \mathbb{E}\!\left[\|f_\theta(\bs^l) - \bx\|_2^2\right] - 2\,\mathbb{E}\!\left[{\bn^r}^\top\!\left(f_\theta(\bs^l) - \bx\right)\right] + \mathbb{E}\!\left[\|\bn^r\|_2^2\right] \\
&= \mathbb{E}\!\left[\|f_\theta(\bs^l) - \bx\|_2^2\right] + \mathbb{E}\!\left[\|\bn^r\|_2^2\right],
\end{align*}
where the last equality follows from $\mathbb{E}\!\left[{\bn^r}^\top f_\theta(\bs^l)\right] = 0$ and $\mathbb{E}\!\left[{\bn^r}^\top \bx\right] = 0$, since $\bn^r$ is zero-mean and independent of both $\bx$ and $\bs^l$. By the same argument with the roles of $\bs^l$ and $\bs^r$ swapped, the second term of~\cref{eq:opt} satisfies
\begin{equation*}
\mathbb{E}\!\left[\|f_\theta(\bs^r) - \bs^l\|_2^2\right] = \mathbb{E}\!\left[\|f_\theta(\bs^r) - \bx\|_2^2\right] + \mathbb{E}\!\left[\|\bn^l\|_2^2\right].
\end{equation*}
Summing the two terms and noting that $\mathbb{E}\!\left[\|\bn^l\|_2^2\right] + \mathbb{E}\!\left[\|\bn^r\|_2^2\right]$ does not depend on $\theta$, we obtain
\begin{equation*}
\argmin_\theta \mathbb{E}\!\left[\|f_\theta(\bs^l) - \bs^r\|_2^2 + \|f_\theta(\bs^r) - \bs^l\|_2^2\right]
= \argmin_\theta \mathbb{E}\!\left[\|f_\theta(\bs^l) - \bx\|_2^2 + \|f_\theta(\bs^r) - \bx\|_2^2\right],
\end{equation*}
which proves the theorem.
\end{proof}

\section{Proof of~\cref{prop_part}}\label{ap:prop_part}
\begin{proposition}
For any signal window \(I^{(k)}\), the partitioning strategy used in iPSD can always achieve a partition that is at least as good as interleaved downsampling \citep{zs-N2N}, in terms of the mismatch between the underlying clean components of the sub-signals \((\bs_{I^{(k),l}}, \bs_{I^{(k),r}})\).
\end{proposition}

\begin{proof}
The fixed interleaved partition used in~\citet{zs-N2N} is one particular case in the iPSD search space, which contains all $\tbinom{W}{W/2}/2$ candidate within-window partitions of $I^{(k)}$. It follows immediately that iPSD can achieve a partition at least as good as the fixed interleaved partition, since the latter is always available as a candidate.

Moreover, the iPSD partition can be strictly better than interleaved downsampling. Consider a clean signal window
\[
\bx_{I^{(k)}} = (1,-1,1,-1).
\]
Under interleaved partitioning, the resulting clean sub-signals are
\[
\bx_{I^{(k),l}} = (1,1), \qquad \bx_{I^{(k),r}} = (-1,-1),
\]
or equivalently in reversed order. These two sub-signals clearly do not share the same underlying clean component. However, iPSD may choose the feasible within-window partition
\[
I^{(k),l}=\{1,2\}, \qquad I^{(k),r}=\{3,4\},
\]
which yields sub-signals with the same underlying signal:
\[
\bx_{I^{(k),l}} = (1,-1), \qquad \bx_{I^{(k),r}} = (1,-1).
\]
This and many other examples, such as those illustrated in \cref{fig:metho} and the real-life recordings in \cref{tab:underlying} of Appendix~\ref{ap:abla_part}, show that for signals with periodic structure, the flexibility of iPSD's search space yields strictly better partitions than fixed interleaved downsampling.
\end{proof}

\section{Justification for the RL reward signal}\label{ap:reward_just}
In this section, we show that under reasonable assumptions, using the converged denoiser loss as RL reward drives $\pi_\psi$ toward $\bs^l$ and $\bs^r$ that are maximally informative of each other's underlying clean signal. For conciseness, we present the $l \to r$ direction; the reverse direction follows symmetrically.

We follow the standard N2N assumption that the noise $\bn$ is zero-mean and independent of the clean signal $\bx$. In the iPSD setting, for the partitioned pair $(\bs^l,\bs^r)$ induced by $\pi_\psi$, we assume that the right-side noise $\bn^r$ remains conditionally zero-mean given the left-side observation $\bs^l$ and the right-side clean signal $\bx^r$, i.e., $\mathbb{E}[\bn^r \mid \bs^l, \bx^r] = \mathbf{0}$. We also assume that the noise power $\mathbb{E}[\|\bn^r\|_2^2]$ does not depend on the partition policy $\pi_\psi$, which is reasonable since the two partitions have fixed cardinality and the noise statistics are approximately homogeneous across samples. For $(\bs^l, \bs^r)\sim \pi_\psi(\cdot\mid \bs)$, the N2N denoising objective can be rewritten as:
\begin{align}
\begin{split}
    \argmin_\theta \mathbb{E}\left[\|f_\theta(\bs^l) - \bs^r\|^2_2\right] 
    &= \argmin_\theta \mathbb{E}\left[\left\|\left(f_\theta(\bs^l) - \mathbb{E}\left[\bs^r\mid \bs^l\right]\right) - \left(\bs^r- \mathbb{E}\left[\bs^r\mid \bs^l\right]\right)\right\|^2_2\right]\\
    &= \argmin_\theta \mathbb{E}\left[\left\|f_\theta(\bs^l) - \mathbb{E}\left[\bs^r\mid \bs^l\right]\right\|^2_2\right] \\
    &\quad - 2\,\mathbb{E}\left[\left(f_\theta(\bs^l) - \mathbb{E}\left[\bs^r\mid \bs^l\right]\right)^\top\left(\bs^r - \mathbb{E}\left[\bs^r\mid \bs^l\right]\right)\right] \\
    &\quad + \mathbb{E}\left[\left\|\bs^r - \mathbb{E}\left[\bs^r\mid \bs^l\right]\right\|^2_2\right] \\
    &\stackrel{(i)}{=} \argmin_\theta \mathbb{E}\left[\left\|f_\theta(\bs^l) - \mathbb{E}\left[\bs^r\mid \bs^l\right]\right\|^2_2\right],
\label{eq_lossl2}
\end{split}
\end{align}
where $(i)$ follows because the third term does not depend on $\theta$ and can be dropped from the $\argmin$, and the cross term vanishes by applying the tower property $\mathbb{E}[X] = \mathbb{E}[\mathbb{E}[X \mid \bs^l]]$:
\begin{align*}
    &\mathbb{E}\left[\left(f_\theta(\bs^l) - \mathbb{E}[\bs^r\mid \bs^l]\right)^\top\left(\bs^r - \mathbb{E}[\bs^r\mid \bs^l]\right)\right] \\
    &\qquad= \mathbb{E}\left[\mathbb{E}\left[\left(f_\theta(\bs^l) - \mathbb{E}[\bs^r\mid \bs^l]\right)^\top\left(\bs^r - \mathbb{E}[\bs^r\mid \bs^l]\right) \;\Big|\; \bs^l\right]\right] \\
    &\qquad\stackrel{(ii)}{=} \mathbb{E}\left[\left(f_\theta(\bs^l) - \mathbb{E}[\bs^r\mid \bs^l]\right)^\top \mathbb{E}\left[\bs^r - \mathbb{E}[\bs^r\mid \bs^l] \;\Big|\; \bs^l\right]\right] \\
    &\qquad= \mathbb{E}\left[\left(f_\theta(\bs^l) - \mathbb{E}[\bs^r\mid \bs^l]\right)^\top \cdot \mathbf{0}\right] = 0,
\end{align*}
where $(ii)$ holds because $f_\theta(\bs^l) - \mathbb{E}[\bs^r\mid \bs^l]$ is a deterministic function of $\bs^l$ that can be pulled out of the conditional expectation given $\bs^l$. The remaining inner factor then vanishes:
\begin{equation*}
    \mathbb{E}\left[\bs^r - \mathbb{E}[\bs^r\mid \bs^l] \;\Big|\; \bs^l\right]
    = \mathbb{E}[\bs^r\mid \bs^l] - \mathbb{E}[\bs^r\mid \bs^l] = \mathbf{0},
\end{equation*}
where we used linearity of conditional expectation together with the fact that $\mathbb{E}[\bs^r\mid \bs^l]$ is itself a deterministic function of $\bs^l$ and therefore equals its own conditional expectation given $\bs^l$. Let $f_{\theta^{\star}}$ denote the optimal denoiser. 

Assuming sufficient expressiveness of the denoising module $f_\theta$, \cref{eq_lossl2} shows that
\begin{equation}
    f_{\theta^{\star}}(\bs^l) = \mathbb{E}\left[\bs^r \mid \bs^l\right].
\label{eq_opt_f}
\end{equation}

Let $\mathcal{L}^\star(\pi_\psi)$ denote the converged N2N loss in the $l \to r$ direction for $(\bs^l,\bs^r)\sim\pi_\psi(\cdot\mid\bs)$. Following~\cref{eq_opt_f},
\begin{align}
\begin{split}
    \mathcal{L}^\star(\pi_\psi) 
    &= \mathbb{E}\left[\|f_{\theta^{\star}}(\bs^l) - \bs^r\|^2_2\right] \\
    &= \mathbb{E}\left[\left\|\bs^r - \mathbb{E}\left[\bs^r \mid \bs^l\right]\right\|^2_2\right] \\
    &= \mathbb{E}\left[\left\|\bx^r + \bn^r - \mathbb{E}\left[\bx^r \mid \bs^l\right] - \mathbb{E}\left[\bn^r \mid \bs^l\right]\right\|^2_2\right] \\
    &\stackrel{(iii)}{=} \mathbb{E}\left[\left\|\left(\bx^r - \mathbb{E}\left[\bx^r \mid \bs^l\right]\right) + \bn^r\right\|^2_2\right] \\
    &\stackrel{(iv)}{=} \mathbb{E}\left[\left\|\bx^r - \mathbb{E}\left[\bx^r \mid \bs^l\right]\right\|^2_2\right] + \mathbb{E}\left[\|\bn^r\|^2_2\right],
\label{eq_converged_loss}
\end{split}
\end{align}
where $(iii)$ uses $\mathbb{E}[\bn^r \mid \bs^l] = \mathbf{0}$, which follows from the tower property: $\mathbb{E}[\bn^r \mid \bs^l] = \mathbb{E}\big[\mathbb{E}[\bn^r \mid \bs^l, \bx^r] \mid \bs^l\big] = \mathbf{0}$ by our assumption. Step $(iv)$ follows because the cross term vanishes. To see this, apply the tower property with conditioning on $(\bs^l, \bx^r)$:
\begin{align*}
    \mathbb{E}\left[\left(\bx^r - \mathbb{E}\left[\bx^r \mid \bs^l\right]\right)^\top \bn^r\right]
    &= \mathbb{E}\left[\mathbb{E}\left[\left(\bx^r - \mathbb{E}\left[\bx^r \mid \bs^l\right]\right)^\top \bn^r \;\Big|\; \bs^l, \bx^r\right]\right] \\
    &= \mathbb{E}\left[\left(\bx^r - \mathbb{E}\left[\bx^r \mid \bs^l\right]\right)^\top \mathbb{E}\left[\bn^r \mid \bs^l, \bx^r\right]\right] \\
    &= \mathbb{E}\left[\left(\bx^r - \mathbb{E}\left[\bx^r \mid \bs^l\right]\right)^\top \cdot \mathbf{0}\right] = 0,
\end{align*}
where the second equality follows because $\bx^r - \mathbb{E}\left[\bx^r \mid \bs^l\right]$ is determined by $\bs^l$ and $\bx^r$, and therefore behaves as a constant under the conditional expectation given $(\bs^l, \bx^r)$; the third uses the assumption that $\mathbb{E}[\bn^r \mid \bs^l, \bx^r] = \mathbf{0}$.

Upon combining the two results, we obtain the following regarding the converged loss of the denoiser:
\begin{equation}
    \mathcal{L}^\star(\pi_\psi) = \underbrace{\mathbb{E}\left[\left\|\bx^r - \mathbb{E}\left[\bx^r \mid \bs^l\right]\right\|^2_2\right]}_{\text{depends on } \pi_\psi} + \underbrace{\mathbb{E}\left[\|\bn^r\|^2_2\right]}_{\text{constant in } \pi_\psi}.
\label{eq_decomp}
\end{equation}
Note that the first term depends on $\pi_\psi$ through the joint distribution of $(\bs^l, \bx^r)$ induced by the partition policy. By our assumption that $\mathbb{E}[\|\bn^r\|^2_2]$ does not depend on $\pi_\psi$, the second term is constant with respect to $\psi$. Therefore, maximizing the RL reward $-\mathcal{L}^\star(\pi_\psi)$ is equivalent to minimizing the first term of \cref{eq_decomp}:
\begin{equation}
    \argmax_\psi \left(-\mathcal{L}^\star(\pi_\psi)\right) = \argmin_\psi \mathbb{E}\left[\left\|\bx^r - \mathbb{E}\left[\bx^r \mid \bs^l\right]\right\|^2_2\right],
\label{eq_argmax_reward}
\end{equation}
which is the minimum mean squared error (MMSE) of predicting the right-side clean signal $\bx^r$ from the left-side observation $\bs^l$. Hence, using the negative converged N2N loss as the RL reward drives $\pi_\psi$ toward partitions under which $\bs^l$ is maximally predictive of $\bx^r$ in the MMSE sense; by symmetry, $\bs^r$ is maximally predictive of $\bx^l$.

\section{Derivation of the policy gradient}\label{ap:grad_as}
We now derive the gradient of the expected denoising reward $\mathbb{E}_{(\bs^l, \bs^r)\sim\pi(\cdot\mid \bs)}\left[R(\bs^l, \bs^r)\right]$ with respect to the parameters $\psi$. Applying the log-derivative identity $\nabla_x f(x) = f(x)\,\nabla_x \log f(x)$ yields
\begin{align}
    \nabla_\psi \mathbb{E}_{(\bs^l, \bs^r)\sim\pi(\cdot\mid \bs)}\left[R(\bs^l, \bs^r)\right] &= \nabla_\psi \sum_{(\bs^l, \bs^r)}\pi_\psi(\bs^l, \bs^r \mid \bs) R(\bs^l, \bs^r) \notag\\
    &= \sum_{(\bs^l, \bs^r)}\pi_\psi(\bs^l, \bs^r \mid \bs) \left[\nabla_\psi \log\pi_\psi(\bs^l, \bs^r \mid \bs) R(\bs^l, \bs^r)\right] \notag\\
    &= \mathbb{E}_{(\bs^l, \bs^r)}\left[\nabla_\psi \log\pi_\psi(\bs^l, \bs^r \mid \bs) R(\bs^l, \bs^r)\right].
\label{eq:pgtheo}
\end{align}
\Cref{eq:pgtheo} is the standard REINFORCE estimator~\citep{williams1992}: the intractable sum over all partitions is replaced by a Monte Carlo expectation, making the gradient estimable from sampled partitions $(\bs^l, \bs^r) \sim \pi_\psi(\cdot \mid \bs)$.

\section{Detailed implementation of iPSD-Zero}
\label{ap:ipsd-zs}
In this section, we provide details of iPSD-Zero. Let arm $a$ denote the partition that corresponds to $(\bs_a^l, \bs_a^r)$. Pulling $a$ leads to a certain reward $R(\bs_a^l, \bs_a^r)$, defined as in~\cref{eq:return} as the negated denoiser $f_\theta$ loss at convergence. Due to the stochastic nature of training neural networks, $R(\bs_a^l, \bs_a^r)$ is also stochastic for each arm. To balance the exploration of less frequently sampled partitions with the exploitation of those that have already demonstrated strong denoising performance, the lil'~UCB index for arm $a$ at round $t$ is defined as
\begin{equation}
  l_a(t) = \hat{\mu}_{a}(t) + (1+\beta)(1+\sqrt{\epsilon})\sqrt{\frac{2\sigma^{2}(1+\epsilon)\log\left(\frac{\log\left((1+\epsilon)T_a(t)\right)}{\delta}\right)}{T_a(t)}},
\label{eq:lilucb}
\end{equation}
where $\hat{\mu}_{a}(t)$ denotes the empirical mean reward of arm $a$ up to round $t$, $T_a(t)$ is the number of times arm $a$ has been pulled up to round $t$, and $\sigma^2$ is the sub-Gaussian variance proxy of the reward distribution. The confidence parameter $\delta \in (0,1)$ and the exploration parameters $\beta, \epsilon > 0$ are user-specified hyperparameters that jointly control the tightness of the confidence interval. The first term in~\cref{eq:lilucb} promotes the exploitation of arms with higher observed rewards, while the second term promotes the exploration of other arms with decreasing intensity as $T_a(t)$ increases. \Cref{alg:lilUCB} summarizes our adaptation of lil'~UCB to the denoising process in the zero-shot setting.

In iPSD-Zero, hyperparameters are set to $\beta=1$, $\alpha = \left[(2+\beta)/\beta\right]^2$, $\sigma=0.008$, $\epsilon=0.01$, and $\delta=0.55\times10^{-4}$. The iPSD-Zero method is initialized by pulling each arm once. We then iteratively apply lil'~UCB to select the next arm. This results in increasingly confident estimates of partition quality for each arm, reflected by the average reward. At convergence, the algorithm identifies the optimal partition, $(\mathbf{s}_{a^\star}^l, \mathbf{s}_{a^\star}^r)$, which is then used to train the final denoiser $f_{\theta^\star}$. The design of iPSD-Zero transforms the zero-shot denoising problem into a structured exploration-exploitation task, where, by leveraging lil'~UCB, we achieve principled and sample-efficient identification of the partition that yields the most reliable self-supervision.

\begin{algorithm}[t]
\caption{iPSD-Zero}
\label{alg:lilUCB}
\begin{algorithmic}[1]
\Require Noisy signal $\mathbf{s}$; hyperparameters $(\alpha,\beta,\sigma,\epsilon,\delta)$; learning rate $\eta^\theta$
\State \textbf{Initialization:} pull each arm once to obtain $\hat{\mu}_a(0)$ for all $a$; set $T_a(0)\gets1$ for all $a$; let $t\gets 0$
\While{$T_a(t) < 1 + \alpha\sum_{i\neq a} T_i(t)$ for all arms $a$}
    \State Compute $l_a(t)$ for all $a$ with \cref{eq:lilucb}
    \State Select arm $a_t = \arg\max_{a} l_a(t)$ and partition $\mathbf{s}$ into $(\mathbf{s}_{a_t}^l,\mathbf{s}_{a_t}^r)$
    \State Initialize $f_\theta$; set $N = 0$
    \While{$\theta$ not converged}
        \State $J_n^d = \bigl\|f_\theta(\mathbf{s}_{a_t}^l)-\mathbf{s}_{a_t}^r\bigr\|_2^2 + \bigl\|f_\theta(\mathbf{s}_{a_t}^r)-\mathbf{s}_{a_t}^l\bigr\|_2^2$
        \State $\theta \gets \theta - \eta^\theta \nabla_\theta J_n^d$;\quad $N = N + 1$
    \EndWhile
    \State Reward $r_t =R(\bs_{a_t}^l, \bs_{a_t}^r) = -\tfrac{1}{10}\sum_{n=N-9}^{\,N} J_n^d$
    \State $\hat{\mu}_a(t+1) \gets \begin{cases} \dfrac{T_{a}(t)\,\hat{\mu}_{a}(t) + r_t}{T_{a}(t)+1}, & a = a_t \\[6pt] \hat{\mu}_a(t), & a \ne a_t \end{cases}$
    \State $T_a(t+1) \gets \begin{cases} T_a(t) + 1, & a = a_t \\ T_a(t), & a \ne a_t \end{cases}$
    \State $t \gets t+1$
\EndWhile
\State $a^\star \gets \arg\max_{a} T_a(t)$
\end{algorithmic}
\end{algorithm}

\section{Notes on iPSD implementation}
\label{ap:net_architectures}
\paragraph{Hardware} The iPSD is implemented in Python 3.12.5 and tested on an Ubuntu 22.04 machine with a 13th Gen Intel Core i7-13850HX CPU, an Nvidia RTX 3500 Ada GPU (12~GB VRAM), and 32~GB of RAM. Compared with the state-of-the-art unsupervised DL baseline~\citep{chen2025self}, iPSD trains in ~5.2 h versus >12 h, uses only $4\%$ of the parameters, and denoises a 10 s EEG segment in 90 ms. Even on a CPU-only machine, iPSD denoises comfortably within real-time constraints due to its lightweight architecture.

\paragraph{Window length}
We evaluate different lengths of the localized window (\cref{tab:abla_win_len}). Since the window length must divide the signal length $L=2560$ and be even (so that $\bs^l$ and $\bs^r$ have equal length), we ablate over window lengths of 4, 8, and 10 samples. Longer windows (e.g., the next feasible length, 16 samples) produce too many candidate partitions, inflating the parameter count of $\pi_\psi$ to an intractable size. Results show that a window length of 8 samples is optimal, providing sufficient flexibility while keeping the search space tractable for efficient policy training ($W=8$ already admits $35$ per window and $11200$ per signal distinct partitions).

\begin{table*}[t]
\caption{Ablation study on localized window lengths.}
\centering
\renewcommand{\arraystretch}{1.1}
\footnotesize
\setlength{\tabcolsep}{2.9pt}
\begin{tabular}{c *{6}{c} | *{6}{c}}
\toprule
& \multicolumn{6}{c|}{\textbf{WGN}}
& \multicolumn{6}{c}{\textbf{EMG}} \\
\cmidrule(lr){2-7}\cmidrule(lr){8-13}
& \multicolumn{3}{c}{$-5$ dB Input SNR}
& \multicolumn{3}{c|}{$0$ dB Input SNR}
& \multicolumn{3}{c}{$-5$ dB Input SNR}
& \multicolumn{3}{c}{$0$ dB Input SNR} \\
\cmidrule(lr){2-4}\cmidrule(lr){5-7}\cmidrule(lr){8-10}\cmidrule(lr){11-13}
{\footnotesize Window length}
  & SNR & PSNR & S-MSE
  & SNR & PSNR & S-MSE
  & SNR & PSNR & S-MSE
  & SNR & PSNR & S-MSE \\
\midrule
4 samples
  & $2.61$ & $16.71$ & $60.03$
  & $5.38$ & $19.48$ & $55.27$
  & $1.78$ & $15.88$ & $60.07$
  & $4.77$ & $18.87$ & $45.13$ \\
8 samples
  & $3.97$  & $18.07$ & $55.14$
  & $5.78$  & $19.90$ & $41.13$
  & $3.63$  & $17.73$ & $53.04$
  & $5.95$  & $20.05$ & $39.21$ \\
10 samples
  & $3.96$ & $18.06$ & $55.47$
  & $5.52$ & $19.62$ & $46.92$
  & $3.37$ & $17.47$ & $56.62$
  & $5.66$ & $19.76$ & $42.12$ \\
\bottomrule
\multicolumn{13}{l}{%
  S-MSE denotes the spectral MSE.}
\end{tabular}
\label{tab:abla_win_len}
\end{table*}

\paragraph{Partitioning module}
Before settling on the architecture of the partitioning model $\pi_\psi$, we compared four candidates: (i) iPSD-MLP, a window-independent baseline without temporal context; (ii) iPSD-U-Net, a multi-scale fully convolutional architecture; (iii) iPSD-RNN, a recurrent neural network capturing sequential context across windows; and (iv) iPSD-GRU, a gated RNN variant with selective retention of long-range context. \Cref{tab:app_all} summarizes the detailed architectures of the four networks. All four variants are trained using the same PPO hyperparameter setting.

\begin{table}[t]
\centering
\footnotesize
\begin{threeparttable}
\caption{Partitioning module architectures}
\label{tab:app_all}
\begin{tabular}{l l c c}
\toprule
\textbf{Model} & \textbf{Layer} & \textbf{Channel Progression} & \textbf{Hidden dim} \\
\midrule
\multirow{3}{*}{iPSD-MLP}
  & FC-1 & $8 \to 64$ & - \\
  & FC-2 & $64 \to 128$ & - \\
  & FC-3 & $128 \to 35$ & - \\
\hline
\multirow{7}{*}{iPSD-U-Net}
  & Conv-Enc-1 & $8 \to 64 \to 64$ & - \\
  & Conv-Enc-2 & $64 \to 128 \to 128$ & - \\
  & Conv-Enc-3 & $128 \to 256 \to 256$ & - \\
  & Conv-Dec-3 & $512 \to 128 \to 128$ & - \\
  & Conv-Dec-2 & $256 \to 64 \to 64$ & - \\
  & Conv-Dec-1 & $128 \to 64 \to 64$ & - \\
  & Conv-Out   & $64 \to 35$ & - \\
\hline
\multirow{5}{*}{iPSD-RNN}
  & RNN-1 & $8 \to 128$ & 64 \\
  & RNN-2 & $128 \to 128$ & 64 \\
  & FC-1  & $128 \to 256$ & - \\
  & FC-2  & $256 \to 256$ & - \\
  & FC-3  & $256 \to 35$ & - \\
\hline
\multirow{5}{*}{iPSD-GRU}
  & GRU-1 & $8 \to 128$ & 64 \\
  & GRU-2 & $128 \to 128$ & 64 \\
  & FC-1  & $128 \to 256$ & - \\
  & FC-2  & $256 \to 256$ & - \\
  & FC-3  & $256 \to 35$ & - \\
\hline
\end{tabular}
\begin{tablenotes}
\footnotesize
\item FC = Fully Connected; Conv-Enc/Dec = 1D Convolutional encoder/decoder block (2 conv layers each).
\item FC and Conv layers use the ReLU activation; Batch normalization is applied in all Conv layers.
\item iPSD-RNN and iPSD-GRU are bidirectional.
\end{tablenotes}
\end{threeparttable}
\end{table}

\Cref{ap_tab_denoisor} compares the denoising performance of all the iPSD variants.
Under WGN contamination, the iPSD-GRU achieved the best overall performance, benefiting from its gating mechanisms, which effectively capture sequential context while mitigating vanishing/exploding gradients. Notably, the other sequential model, iPSD-RNN, performed worse than iPSD-MLP, which does not use any sequential context. This indicates that while sequential context is indeed important for partitioning $\bs$, long sequences must be handled appropriately, which GRU achieves successfully. The iPSD-U-Net model does not perform well compared to the other variants, which may be attributed to its relatively limited receptive field when applied to long EEG signals. Based on these results, we chose GRU as the architecture for the partitioning model.

\begin{table}[htbp]
\newcommand{\std}[1]{{\scriptstyle \pm #1}}
\caption{Denoising performance of different partitioning module architectures.}
\centering
\footnotesize
\begin{tabular}{l c c c c c c}
\toprule
& \multicolumn{3}{c}{-5 dB Input SNR}
  & \multicolumn{3}{c}{0 dB Input SNR} \\ \cmidrule(lr){2-7}
Method & SNR & PSNR & S-MSE
  & SNR & PSNR & S-MSE \\ \midrule
\multicolumn{7}{c}{WGN artifacts} \\ \midrule
 iPSD-MLP
   & $3.88$ & $17.97$ & $63.66$
   & $5.73$ & $19.83$ & $54.11$ \\
 iPSD-U-Net
   & $3.76$ & $17.87$ & $69.07$
   & $5.69$ & $19.79$ & $52.46$ \\
 iPSD-RNN
   & $3.83$ & $17.94$ & $65.89$
   & $5.71$ & $19.81$ & $51.24$ \\
 iPSD-GRU
   & $3.97$ & $18.07$ &$55.14$
   & $5.78$ & $19.90$ &$41.13$ \\
\midrule
\multicolumn{7}{c}{EMG artifacts} \\ \midrule
 iPSD-MLP
   & $3.19$ & $17.29$ & $58.46$
   & $5.23$ & $19.33$ & $49.48$ \\
 iPSD-U-Net
   & $3.03$ & $17.13$ & $59.34$
   & $5.18$ & $19.29$ & $50.01$ \\
 iPSD-RNN
   & $2.94$ & $17.04$ & $61.30$
   & $5.05$ & $19.15$ & $51.55$ \\
 iPSD-GRU
   & $3.63$ & $17.73$ & $53.04$
   & $5.95$ & $20.05$ & $39.21$\\
\hline
\end{tabular}
\label{ap_tab_denoisor}
\end{table}

\section{Baselines}\label{ap:baselines}

\begin{table}[t]
\centering
\footnotesize
\caption{Baselines used in quantitative comparisons}
\label{tab:baselines}
\newcommand{\desc}[1]{\parbox[t]{0.4\linewidth}{\raggedright #1\strut}}
\begin{tabular}{@{}ll@{}}
\toprule
\textbf{Method} & \textbf{Description} \\
\midrule
mVMD~\citep{dora2020correlation} & \desc{Modified Variational Mode Decomposition for correlated noise} \\
CEEMDAN~\citep{colominas2014improved} & \desc{Ensemble Empirical Mode Decomposition with noise injection} \\
FrWT~\citep{houamed2020ecg} & \desc{Thresholding of fractional Wavelet Transform coefficients} \\
EMD-LoG~\citep{ranjan2022motion} & \desc{Laplacian-of-Gaussian filtering of Empirical Mode Decomposition components} \\
WPT-ICA~\citep{kerechanin2022eeg} & \desc{Wavelet Packet Transform and Independent Component Analysis} \\
Optimal-WT~\citep{alyasseri2019eeg} & \desc{Optimized thresholding of the Discrete Wavelet Transform} \\
\citeauthor{chen2025self}~\citep{chen2025self} & \desc{Self-supervised consistency learning with masked reconstruction} \\
\bottomrule
\end{tabular}
\end{table}

\Cref{tab:baselines} provides the details of the baselines used in our quantitative evaluation. These baselines cover the dominant paradigms in EEG artifact removal, spanning wavelet-based methods such as FrWT by \citet{houamed2020ecg} and Optimal-WT by \citet{alyasseri2019eeg}; mode decomposition methods such as mVMD by \citet{dora2020correlation}, CEEMDAN by \citet{colominas2014improved}, and EMD-LoG by \citet{ranjan2022motion}; hybrid methods such as WPT-ICA by \citet{kerechanin2022eeg}; and self-supervised DL methods by \citet{chen2025self}. All baselines are implemented using the authors' released code where available, otherwise re-implemented from the published descriptions. Hyperparameters follow each method's original specification. For CEEMDAN, we use the PyEMD implementation with default parameters and discard the first 2 intrinsic mode functions. For \citeauthor{chen2025self}, we retrain it on the same training set that iPSD uses with the authors' default hyperparameters.

\begin{table}[t]
\newcommand{\std}[1]{{\scriptstyle \pm #1}}
\newcommand{\smse}[2]{(#1{\scriptstyle\,\pm\,#2}){\scriptscriptstyle \times 10^{4}}}
\newcommand{\smsea}[2]{(#1{\scriptstyle\,\pm\,#2}){\scriptscriptstyle \times 10^{3}}}
\caption{Standard deviation of denoising performance. Values are reported as mean $\pm$ std.}
\centering
\footnotesize
\setlength{\tabcolsep}{3.5pt}
\renewcommand{\arraystretch}{1.1}
\begin{tabular}{cl ccc ccc}
\toprule
& Method
  & \multicolumn{3}{c}{-5 dB Input SNR}
  & \multicolumn{3}{c}{0 dB Input SNR} \\ \cmidrule(lr){3-5} \cmidrule(lr){6-8}
& & SNR & PSNR & S-MSE
  & SNR & PSNR & S-MSE \\ \midrule
\multirow{8}{*}{\rotatebox[origin=c]{90}{\shortstack{WGN}}}
& mVMD
   & $-5.05\std{1.93}$ & $9.05\std{3.67}$ & $\smse{1.46}{0.31}$
   & $2.18\std{2.85}$  & $16.28\std{3.62}$ & $\smse{1.13}{0.36}$ \\
& CEEMDAN
   & $-3.85\std{1.61}$ & $10.25\std{3.31}$ & $\smse{0.98}{0.24}$
   & $4.13\std{1.84}$  & $18.23\std{2.67}$ & $\smse{0.78}{0.25}$ \\
& FrWT
   & $-3.54\std{1.41}$ & $10.56\std{2.91}$ & $\smse{0.96}{0.21}$
   & $3.43\std{1.92}$  & $17.53\std{2.73}$ & $\smse{0.83}{0.27}$ \\
& EMD-LoG
   & $2.28\std{1.79}$ & $16.38\std{3.58}$  & $\smse{0.80}{0.17}$
   & $4.40\std{1.87}$  & $18.50\std{2.65}$ & $\smse{0.80}{0.25}$ \\
& WPT-ICA
   & $-4.89\std{1.87}$ & $9.21\std{2.90}$  & $\smse{1.63}{0.28}$
   & $1.74\std{1.55}$  & $15.84\std{2.46}$ & $\smse{1.23}{0.25}$ \\
& \citeauthor{chen2025self}
   & $2.01\std{1.28}$ & $16.11\std{2.52}$ & $\smse{0.28}{0.07}$
   & $3.25\std{1.67}$ & $17.35\std{2.69}$ & $\smse{0.21}{0.04}$ \\
& Optimal-WT
   & $3.34\std{1.50}$  & $17.44\std{3.18}$ & $\smse{0.36}{0.11}$
   & $5.07\std{1.72}$  & $19.18\std{2.73}$ & $\smse{0.27}{0.11}$ \\
& iPSD
   & \cellcolor{myred!20}{$3.97\std{0.64}$} & \cellcolor{myred!20}{$18.07\std{2.85}$} & \cellcolor{myred!20}{$55.14\std{27.62}$}
   & \cellcolor{myred!20}{$5.78\std{1.25}$} & \cellcolor{myred!20}{$19.90\std{2.38}$} & \cellcolor{myred!20}{$41.13\std{18.33}$} \\
& iPSD-Zero
   & \cellcolor{myblue!20}{$3.89\std{0.79}$} & \cellcolor{myblue!20}{$18.00\std{2.52}$} & \cellcolor{myblue!20}{$58.34\std{25.13}$}
   & \cellcolor{myblue!20}{$5.76\std{1.32}$} & \cellcolor{myblue!20}{$19.86\std{2.60}$} & \cellcolor{myblue!20}{$48.67\std{18.65}$} \\
\hline
\multirow{8}{*}{\rotatebox[origin=c]{90}{\shortstack{EMG}}}
&  mVMD
   & $-6.05\std{2.31}$ & $8.05\std{3.45}$ & $\smse{5.40}{1.14}$
   & $0.93\std{2.24}$ & $15.03\std{3.25}$ & $\smse{3.06}{0.97}$ \\
&  CEEMDAN
   & $0.25\std{1.31}$ & $14.35\std{2.75}$ & $\smsea{0.36}{0.14}$
   & $3.26\std{2.01}$ & $17.36\std{3.02}$ & $\smsea{0.22}{0.08}$ \\
&  FrWT
   & $-0.06\std{1.42}$ & $14.04\std{3.21}$ & $\smsea{0.55}{0.17}$
   & $3.53\std{1.43}$ & $17.63\std{3.12}$ & $\smsea{0.62}{0.21}$ \\
&  EMD-LoG
   & $0.32\std{1.63}$ & $14.42\std{3.20}$ & $\smsea{0.38}{0.13}$
   & $3.55\std{1.58}$ & $17.65\std{3.31}$ & $\smsea{0.22}{0.11}$ \\
&  WPT-ICA
   & $-0.09\std{1.22}$ & $14.01\std{2.47}$ & $\smsea{0.66}{0.13}$
   & $2.36\std{2.07}$ & $16.46\std{3.59}$ & $\smsea{0.31}{0.08}$ \\
& \citeauthor{chen2025self}
   & $-0.05\std{0.18}$ & $14.05\std{2.56}$ & $\smsea{0.42}{0.14}$
   & $3.01\std{1.77}$ & $17.12\std{3.23}$ & $\smsea{0.19}{0.09}$ \\
&  Optimal-WT
   & $0.33\std{2.40}$ & $14.42\std{3.69}$ & $\smsea{0.69}{0.14}$
   & $3.89\std{1.95}$ & $18.10\std{3.46}$ & $\smsea{0.43}{0.10}$ \\
&  iPSD
   & \cellcolor{myred!20}{$3.63\std{1.16}$} & \cellcolor{myred!20}{$17.73\std{2.44}$} & \cellcolor{myred!20}{$53.04\std{29.91}$}
   & \cellcolor{myred!20}{$5.95\std{1.31}$} & \cellcolor{myred!20}{$20.05\std{2.55}$} & \cellcolor{myred!20}{$39.21\std{17.17}$} \\
&  iPSD-Zero
   & \cellcolor{myblue!20}{$3.27\std{1.14}$} & \cellcolor{myblue!20}{$17.37\std{2.32}$} & \cellcolor{myblue!20}{$56.62\std{29.75}$}
   & \cellcolor{myblue!20}{$5.31\std{1.22}$} & \cellcolor{myblue!20}{$19.41\std{2.65}$} & \cellcolor{myblue!20}{$44.12\std{17.68}$} \\
\hline
\multicolumn{8}{l}{* \colorbox{myred!20}{\raisebox{0pt}[0.5em][0em]{Pink}} (\colorbox{myblue!20}{\raisebox{0pt}[0.5em][0em]{blue}}) highlights the best (second-best) performance;\, S-MSE stands for spectral MSE.}
\end{tabular}
\label{tab:snr_comp_full}
\end{table}

\section{Evaluation metrics}
\label{ap:metrics}

\paragraph{\small Signal-to-Noise Ratio (SNR)}
SNR is defined as the ratio of the clean signal power to the residual noise power after denoising. A higher SNR indicates better noise suppression while retaining the power of the original signal. The SNR (in dB) is computed as
\begin{equation}
\mathrm{SNR} = 10 \log_{10} \left( \frac{\|\mathbf{x}\|_2^2}{\|\mathbf{x}-\hat{\mathbf{x}}\|_2^2} \right),
\end{equation}
where $\mathbf{x}$ denotes the clean reference signal and $\hat{\mathbf{x}}$ the denoising output. This metric directly measures reconstruction accuracy in the time domain.

\paragraph{\small Peak Signal-to-Noise Ratio (PSNR)}
PSNR is defined as the ratio of the squared peak amplitude of the clean signal to the average power of the residual noise after denoising. A higher PSNR indicates better noise suppression while retaining extreme values of the original signal, which are often clinically relevant information for EEG signals (e.g., spikes or sharp waves). The PSNR (in dB) is defined as
\begin{align}
\begin{split}
    \mathrm{PSNR} = 10 \log_{10} \left( \frac{x_{\max}^2}{\|\mathbf{x}-\hat{\mathbf{x}}\|_2^2/L} \right),
\end{split}
\end{align}
where $L$ is the signal length and $x_{\max}$ is the maximum amplitude of the clean signal. Compared to SNR, which focuses on average power, PSNR is more sensitive to extreme values in the signal.

\paragraph{\small Mean Squared Error of the Spectrum (Spectral MSE)}
Since the clinical interpretation of EEG signals often relies on frequency-domain information (e.g., power in $\delta$, $\theta$, $\alpha$, $\beta$, and $\gamma$ bands)~\citep{vallat2021open}, we also measure reconstruction accuracy in the frequency domain. Spectral MSE is defined as the mean squared error between PSDs of the clean and denoised signals. The PSD of a signal $\mathbf{x}$ (in dB) is defined as
\begin{equation}
\mathrm{spec}_{\mathbf{x}}(f) = 10 \log_{10}\left( \sum_{m=-\infty}^{\infty}R_{\mathbf{x}\mathbf{x}}[m] \rm{e}^{-\rm{j}2\pi fm}\right) ,
\end{equation}
where $R_{\mathbf{x}\mathbf{x}}[m] = \mathbb{E}\left [ x_{i} \cdot x_{i+m}\right ]$ denotes the autocorrelation of $\mathbf{x}$, and $f$ denotes the frequency. In practice, we estimate the PSD of finite-length signals using Welch's method, which returns discrete frequency bins $\{f_k\}_{k=1}^{N_f}$. The Spectral MSE is then computed as
\begin{equation}
\mathrm{Spectral\ MSE} = \frac{1}{N_f} \sum_{k=1}^{N_f}
\bigl( \mathrm{spec}_{\mathbf{x}}(f_k) -
\mathrm{spec}_{\hat{\mathbf{x}}}(f_k) \bigr)^{2}.
\end{equation}
This metric assesses how well the denoising algorithm removes noise while preserving the spectral distribution of the EEG signal, which is critical for downstream EEG analysis.

\section{Additional results}
\label{ap:denoise_std}

\subsection{Comparison of denoising performance} \label{ap:std_all}
In \cref{tab:snr_comp_full}, we report the standard deviation across all EEG segments for the comparison in \cref{tab:mean} of \cref{Sec:experiments}. We observe that both iPSD and iPSD-Zero significantly outperform baselines on average, while remaining highly consistent across segments, demonstrating the robustness of the proposed framework.

\subsection{Independent runs with different random seeds}\label{ap:random_seed}
For iPSD, we conduct independent training runs across 5 distinct random seeds and report the mean $\pm$ std under different noise conditions in \cref{tab:rand_seed}. Across the runs, SNR and PSNR have std $\leq 0.06$ dB and Spectral MSE has std $<1.5\%$, indicating that iPSD consistently converges to high performance, with a clear gain over the interleaved downsampling baseline (by $1.41$ dB higher SNR and $26.6\%$ lower S-MSE on average). We also observe no degenerate partitions, as the entropies of selected partitions across all runs lie in the $[0.5, 0.75]$ range, far from uniform (entropy$=1$) or deterministic~(entropy$=0$).

\begin{table}[t]
  \newcommand{\std}[1]{{\scriptstyle \pm #1}}
  \renewcommand{\arraystretch}{1.1}
  \centering
  \footnotesize
  \caption{Stability of iPSD convergence. Values are reported as mean $\pm$ std over independent training runs across 5 distinct random seeds.}
  \label{tab:rand_seed}
  \begin{tabular}{cccccc}
    \toprule
    & Input SNR & SNR & PSNR & S-MSE & Partition entropy \\
    \midrule
    \multirow{2}{*}{\rotatebox[origin=c]{90}{WGN}}
      & $-5$ dB & $3.96\std{0.05}$ & $18.06\std{0.05}$ & $55.26\std{0.36}$ & $0.58\std{0.03}$ \\
      & $0$ dB  & $5.75\std{0.06}$ & $19.86\std{0.06}$ & $41.30\std{0.51}$ & $0.53\std{0.04}$ \\
    \midrule
    \multirow{2}{*}{\rotatebox[origin=c]{90}{EMG}}
      & $-5$ dB & $3.64\std{0.06}$ & $17.74\std{0.06}$ & $53.01\std{0.62}$ & $0.74\std{0.02}$ \\
      & $0$ dB  & $5.94\std{0.03}$ & $20.04\std{0.03}$ & $39.42\std{0.23}$ & $0.69\std{0.02}$ \\
    \bottomrule
  \end{tabular}
\end{table}

\subsection{Ablation on intelligent partition} \label{ap:abla_part}

\begin{table}[t]
\newcommand{\std}[1]{{\scriptstyle \pm #1}}
\renewcommand{\arraystretch}{1.1}
\caption{Ablation on the iPSD partitioning strategy. Values are reported as mean $\pm$ std.}
\centering
\footnotesize
\begin{tabular}{cl ccc ccc}
\toprule
& & \multicolumn{3}{c}{-5 dB Input SNR}
    & \multicolumn{3}{c}{0 dB Input SNR} \\ \cmidrule(lr){3-8}
& Method & SNR & PSNR & S-MSE
    & SNR & PSNR & S-MSE \\ \midrule
\multirow{4}{*}{\rotatebox[origin=c]{90}{\shortstack{WGN}}}
 & ID
   & $2.58\std{0.95}$ & $16.68\std{2.92}$ & $62.33\std{37.51}$
   & $5.13\std{1.79}$ & $19.24\std{3.27}$ & $67.18\std{37.21}$ \\
& DR
   & $2.74\std{0.94}$ & $16.84\std{2.89}$ & $59.43\std{33.21}$
   & $5.30\std{1.19}$ & $19.41\std{3.05}$ & $53.62\std{24.73}$ \\
 & iPSD
   & \cellcolor{myred!20}{$3.97\std{0.64}$} & \cellcolor{myred!20}{$18.07\std{2.85}$} & \cellcolor{myred!20}{$55.14\std{27.62}$}
   & \cellcolor{myred!20}{$5.78\std{1.25}$} & \cellcolor{myred!20}{$19.90\std{2.38}$} & \cellcolor{myred!20}{$41.13\std{18.33}$} \\
 & iPSD-Zero
   & \cellcolor{myblue!20}{$3.89\std{0.79}$} & \cellcolor{myblue!20}{$18.00\std{2.52}$} & \cellcolor{myblue!20}{$58.34\std{25.13}$}
   & \cellcolor{myblue!20}{$5.76\std{1.32}$} & \cellcolor{myblue!20}{$19.86\std{2.60}$} & \cellcolor{myblue!20}{$48.67\std{18.65}$} \\
\hline
\multirow{4}{*}{\rotatebox[origin=c]{90}{\shortstack{EMG}}}
 & ID
   & $1.62\std{1.00}$ & $15.73\std{2.78}$ & $67.68\std{25.53}$
   & $4.31\std{1.57}$ & $18.41\std{3.07}$ & $60.52\std{35.34}$ \\
& DR
   & $2.20\std{1.01}$ & $16.31\std{2.88}$ & $63.88\std{32.73}$
   & $4.73\std{1.28}$ & $18.83\std{2.96}$ & $52.27\std{22.11}$ \\
 & iPSD
   & \cellcolor{myred!20}{$3.63\std{1.16}$} & \cellcolor{myred!20}{$17.73\std{2.44}$} & \cellcolor{myred!20}{$53.04\std{29.91}$}
   & \cellcolor{myred!20}{$5.95\std{1.31}$} & \cellcolor{myred!20}{$20.05\std{2.55}$} & \cellcolor{myred!20}{$39.21\std{17.17}$} \\
 & iPSD-Zero
   & \cellcolor{myblue!20}{$3.27\std{1.14}$} & \cellcolor{myblue!20}{$17.37\std{2.32}$} & \cellcolor{myblue!20}{$56.62\std{29.75}$}
   & \cellcolor{myblue!20}{$5.31\std{1.22}$} & \cellcolor{myblue!20}{$19.41\std{2.65}$} & \cellcolor{myblue!20}{$44.12\std{17.68}$} \\
\hline
\multicolumn{8}{l}{* \colorbox{myred!20}{\raisebox{0pt}[0.5em][0em]{Pink}} (\colorbox{myblue!20}{\raisebox{0pt}[0.5em][0em]{blue}}) highlights the best (second-best) performance.}
\end{tabular}
\label{tab:abla_partition_full}
\end{table}

In \cref{tab:abla_partition} of \cref{Sec:experiments}, we ablate the partitioning module by switching the $\pi_\psi$ of iPSD to simpler heuristics. In ID, we assign alternating samples to the two sub-signals, which was proven effective for locally smooth images \citep{zs-N2N}. In DR, we maintain $W$ continuous, differentiable Gumbel-Softmax logits for the samples in each window, denoted by $p_{\text{soft}}$. We assign the top $W/2$ samples to one sub-signal and the remainder to the other, yielding a binary assignment $p_{\text{hard}} \in \{0,1\}^W$. We then use the straight-through estimator in \citet{bengio2013estimating}:
\begin{align}
    p = p_{\text{hard}} + p_{\text{soft}} - \operatorname{stopgrad}(p_{\text{soft}}).
\end{align}
This means the forward pass uses discrete assignments $p_{\text{hard}}$, while gradients propagate through $p_{\text{soft}}$.

In \cref{tab:abla_partition_full}, we report the standard deviation across all EEG segments for the ablation in \cref{tab:abla_partition}, which shows that the intelligent partitioning is consistently effective across EEG segments. Notably, DR gains only $~3.5\%$ speed increase compared to RL, but at significant cost to performance. This is because DR optimizes a biased approximation of an otherwise non-smooth objective, while RL optimizes a smooth, differentiable objective by construction.

Additionally, we show that the intelligent iPSD partitioning generates sub-signals that better support self-supervised denoising by computing signal mismatch and noise correlation, which directly correspond to the central assumptions of \cref{theorem:eq} that determine denoising performance. On synthetic signals, the underlying signal $(\bx^l, \bx^r)$ and noise $(\bn^l, \bn^r)$ are recovered by applying the partition indices to $\bx$ and $\bn$. Signal mismatch is measured by the normalized mean squared error (NMSE) between $\bx^l$ and $\bx^r$, where a lower value indicates more similar underlying signals in the two sub-signals. Noise correlation is measured by the absolute Pearson correlation between $\bn^l$ and $\bn^r$, where a lower value indicates more independent noise in the two sub-signals. As shown in \cref{tab:underlying}, across all noise conditions, iPSD and iPSD-Zero generate sub-signals with more similar underlying signals and less correlated noise, leading to their strong denoising performance in \cref{tab:abla_partition_full}.

\begin{table}[t]
\newcommand{\std}[1]{{\scriptstyle \pm #1}}
\renewcommand{\arraystretch}{1.1}
\caption{Comparing the quality of sub-signals $(\bs^l, \bs^r)$ generated by different partition strategies. Values are reported as mean $\pm$ std, scaled by $10^{2}$.}
\centering
\footnotesize
\begin{tabular}{cl cc cc}
\toprule
& & \multicolumn{2}{c}{-5 dB Input SNR}
    & \multicolumn{2}{c}{0 dB Input SNR} \\ \cmidrule(lr){3-4} \cmidrule(lr){5-6}
& Method & Signal mismatch & Noise correlation
    & Signal mismatch & Noise correlation \\ \midrule
\multirow{4}{*}{\rotatebox[origin=c]{90}{\shortstack{WGN}}}
 & ID
   & $32.52\std{24.29}$ & $2.22\std{1.55}$
   & $32.52\std{24.29}$ & $2.22\std{1.55}$ \\
& DR
   & $22.32\std{16.62}$ & $2.28\std{1.61}$
   & $33.27\std{20.53}$ & $2.79\std{2.38}$ \\
 & iPSD
   & \cellcolor{myred!20}{$13.04\std{11.81}$} & \cellcolor{myred!20}{$1.47\std{0.70}$}
   & \cellcolor{myred!20}{$14.95\std{12.30}$} & \cellcolor{myred!20}{$1.26\std{0.60}$} \\
 & iPSD-Zero
   & \cellcolor{myblue!20}{$14.32\std{13.18}$} & \cellcolor{myblue!20}{$1.80\std{1.72}$}
   & \cellcolor{myblue!20}{$15.81\std{13.28}$} & \cellcolor{myblue!20}{$1.31\std{0.91}$} \\
\hline
\multirow{4}{*}{\rotatebox[origin=c]{90}{\shortstack{EMG}}}
 & ID
   & $32.52\std{24.29}$ & $12.35\std{7.87}$
   & $32.52\std{24.29}$ & $12.35\std{7.87}$ \\
& DR
   & $34.43\std{15.01}$ & $14.16\std{8.69}$
   & $28.81\std{16.18}$ & $11.23\std{11.08}$ \\
 & iPSD
   & \cellcolor{myred!20}{$16.23\std{12.81}$} & \cellcolor{myred!20}{$2.25\std{0.87}$}
   & \cellcolor{myred!20}{$16.40\std{12.94}$} & \cellcolor{myred!20}{$2.44\std{0.82}$} \\
 & iPSD-Zero
   & \cellcolor{myblue!20}{$18.72\std{15.32}$} & \cellcolor{myblue!20}{$3.16\std{0.91}$}
   & \cellcolor{myblue!20}{$18.81\std{13.90}$} & \cellcolor{myblue!20}{$2.70\std{0.93}$} \\
\hline
\multicolumn{6}{l}{* \colorbox{myred!20}{\raisebox{0pt}[0.5em][0em]{Pink}} (\colorbox{myblue!20}{\raisebox{0pt}[0.5em][0em]{blue}}) highlights the best (second-best) performance.}
\end{tabular}
\label{tab:underlying}
\end{table}

\subsection{Applicability to lower sampling frequency}\label{ap:lower_fs}
Many real-world EEG signals are recorded at a lower sampling frequency with reduced local smoothness. To evaluate the robustness of iPSD under such settings, we use the same synthetic dataset in the main text, but downsample it from 256 Hz to 128 Hz. \Cref{tab:lower_fs} compares iPSD against the ID baseline in terms of output SNR. Because interleaved downsampling hinges on local smoothness, we observe a large performance degradation under low sampling frequency, whereas iPSD remains robust, falling by less than 0.07 dB across all noise conditions. On average, iPSD's gains over ID are 1.96 dB at 128 Hz, compared to 1.42 dB at 256 Hz. The results corroborate that flexible partitioning overcomes the strong local smoothness assumption of simpler heuristics, such as ID, and allows splits with higher underlying signal similarity.

\begin{table}[t]
  \newcommand{\std}[1]{{\scriptstyle \pm #1}}
  \renewcommand{\arraystretch}{0.9}
  \centering
  \footnotesize
  \caption{Output SNR of iPSD and Interleaved Downsampling (ID) on EEG signals with lower sampling rate.}
  \label{tab:lower_fs}
  \begin{tabular}{llcccc}
    \midrule
    & & \multicolumn{2}{c}{WGN} & \multicolumn{2}{c}{EMG} \\
    \cmidrule(lr){3-4} \cmidrule(lr){5-6}
    Method & Sampling rate & $-5$ dB & $0$ dB & $-5$ dB & $0$ dB \\
    \midrule
    \multirow{2}{*}{iPSD} & 128 Hz & $3.99\std{0.70}$ & $5.76\std{1.23}$ & $3.56\std{1.12}$ & $5.89\std{1.27}$ \\
                          & 256 Hz & $3.97\std{0.64}$ & $5.78\std{1.25}$ & $3.63\std{1.16}$ & $5.95\std{1.31}$ \\
    \midrule
    \multirow{2}{*}{ID}   & 128 Hz & $2.32\std{0.90}$ & $4.85\std{1.23}$ & $0.21\std{1.63}$ & $3.99\std{1.44}$ \\
                          & 256 Hz & $2.58\std{0.95}$ & $5.13\std{1.79}$ & $1.62\std{1.00}$ & $4.31\std{1.57}$ \\
    \hline
  \end{tabular}
\end{table}

\subsection{Supervised learning baselines}\label{ap:comp_supervised}
Since there is no other ground-truth-free DL denoiser that we are aware of, we compare iPSD with supervised DL methods that, in a sense, indicate the upper limit to self-supervised denoising performance. The baselines are DuoCL~\citep{gao2022eeg} and DenoiseFormer~\citep{chen2023denosieformer}, which rely on highly expressive DL architectures. The results in Table~\ref{tab:cmp_supervised} show that, without the clean-reference supervision that these supervised baselines hinge on, iPSD comes within 0.6 dB of their performance, all while using far fewer parameters. We further train the denoiser $f_\theta$ of iPSD with clean signals, which acts as an architecture-matched upper bound. On WGN, supervised training yields $4.17$ dB and $6.03$ dB at $-5$ dB and $0$ dB. On EMG, supervised training yields $3.81$ dB and $6.11$ dB at $-5$ dB and $0$ dB. iPSD comes within 0.25 dB of this upper bound despite training with no access to the ground-truths.

\begin{table}[t]
\newcommand{\std}[1]{{\scriptstyle \pm #1}}
\caption{Comparing iPSD (trained without any ground-truth) and DL EEG denoisers (trained with full supervision using ground-truth signals).}
    \centering
    \footnotesize
    \setlength{\tabcolsep}{4pt}
    \renewcommand{\arraystretch}{1.0}
    \begin{tabular}{lcccccc}
        \toprule
        & & & \multicolumn{2}{c}{WGN} & \multicolumn{2}{c}{EMG} \\
        \cmidrule(lr){4-5}\cmidrule(lr){6-7}
        Method & GT-free & Params & $-5$ dB & $0$ dB & $-5$ dB & $0$ dB \\
        \midrule
        DuoCL & \ding{55} & 8.67 m & $4.12\std{1.31}$ & $5.90\std{1.55}$ & $3.71\std{1.21}$ & $6.02\std{1.58}$ \\
        DenoiseFormer & \ding{55} & 6.90 m & $4.24\std{1.25}$ & $6.36\std{1.43}$ & $3.94\std{1.00}$ & $6.43\std{1.34}$ \\
        iPSD & \ding{51} & 0.33 m & $3.97\std{0.64}$ & $5.78\std{1.25}$ & $3.63\std{1.16}$ & $5.95\std{1.31}$ \\
        \hline
    \end{tabular}
    \label{tab:cmp_supervised}
\end{table}

\subsection{Robustness against AR noise}\label{ap:assump}
Although \cref{theorem:eq} assumes that the two partitioned views are independent noisy realizations of the same underlying signal, iPSD can effectively recover the clean signal even when this assumption is not strictly satisfied. For example, iPSD achieves robust denoising performance in our main experiments despite the temporal correlation in the EMG noise. To systematically assess its sensitivity to temporal correlation and investigate potential failure modes, we conduct a controlled experiment using AR(1) noise at an input SNR of $-5$ dB across a range of autoregressive coefficients $\rho$. 
From \cref{tab:ar1_noise}, we observe the performance (indicated by SNR value) change with increasing absolute temporal correlation; the proposed iPSD retains competitive performance even in the presence of markedly correlated noise ($\rho = \pm 0.6$). Compared with the interleaved downsampling (ID) baseline, iPSD achieves an even larger performance advantage under more correlated noise. This is due to the flexible partitioning of iPSD, making it more robust to violations of the independence assumption compared to the ablated ID counterpart without the learnable partitioning component.

\begin{table}[t]
\newcommand{\std}[1]{{\scriptstyle \pm #1}}
\renewcommand{\arraystretch}{1.1}
\centering
\caption{Output SNR (dB) under AR(1) noise at an input SNR of $-5$ dB for different autoregressive coefficients $\rho$. $\Delta$ stands for the improvement of iPSD over ID.}
\label{tab:ar1_noise}
\setlength{\tabcolsep}{4pt}
\footnotesize
\begin{tabular}{lccccccc}
\toprule
$\rho$ & $-0.6$ & $-0.4$ & $-0.2$ & $0$ & $0.2$ & $0.4$ & $0.6$ \\
\midrule
ID & $0.95\std{0.93}$ & $1.89\std{0.92}$ & $2.44\std{0.88}$ & $2.58\std{0.95}$ & $2.42\std{0.83}$ & $1.70\std{0.70}$ & $0.73\std{0.60}$ \\
iPSD & $3.23\std{0.93}$ & $3.69\std{0.87}$ & $3.94\std{0.71}$ & $3.97\std{0.64}$ & $3.94\std{0.69}$ & $3.87\std{0.83}$ & $3.74\std{0.95}$ \\
\hline
$\Delta$ & $+2.28$ & $+1.80$ & $+1.50$ & $+1.39$ & $+1.52$ & $+2.17$ & $+3.01$ \\
\hline
\end{tabular}
\end{table}

\section{Visual comparisons between iPSD and Optimal-WT}\label{ap:syn_qual}
\Cref{fig:syn_perf} uses example synthetic EEG segments to illustrate the denoising performance of both iPSD and iPSD-Zero. For reference, we also include the strongest baseline, Optimal-WT. The input signals are corrupted with synthetically added WGN or EMG artifacts, with input SNR levels set to either $-5$ dB or $0$ dB. For each input signal, we show the noisy input EEG (orange), the clean ground-truth EEG (green), the Optimal-WT output (gray), the iPSD output (pink), and the iPSD-Zero output (blue). The lower right panel for each signal displays the PSDs of the original and denoised signals, allowing visual comparison of all the methods in the frequency domain.

\begin{figure}[t]
\centering
\includegraphics[width=\textwidth]{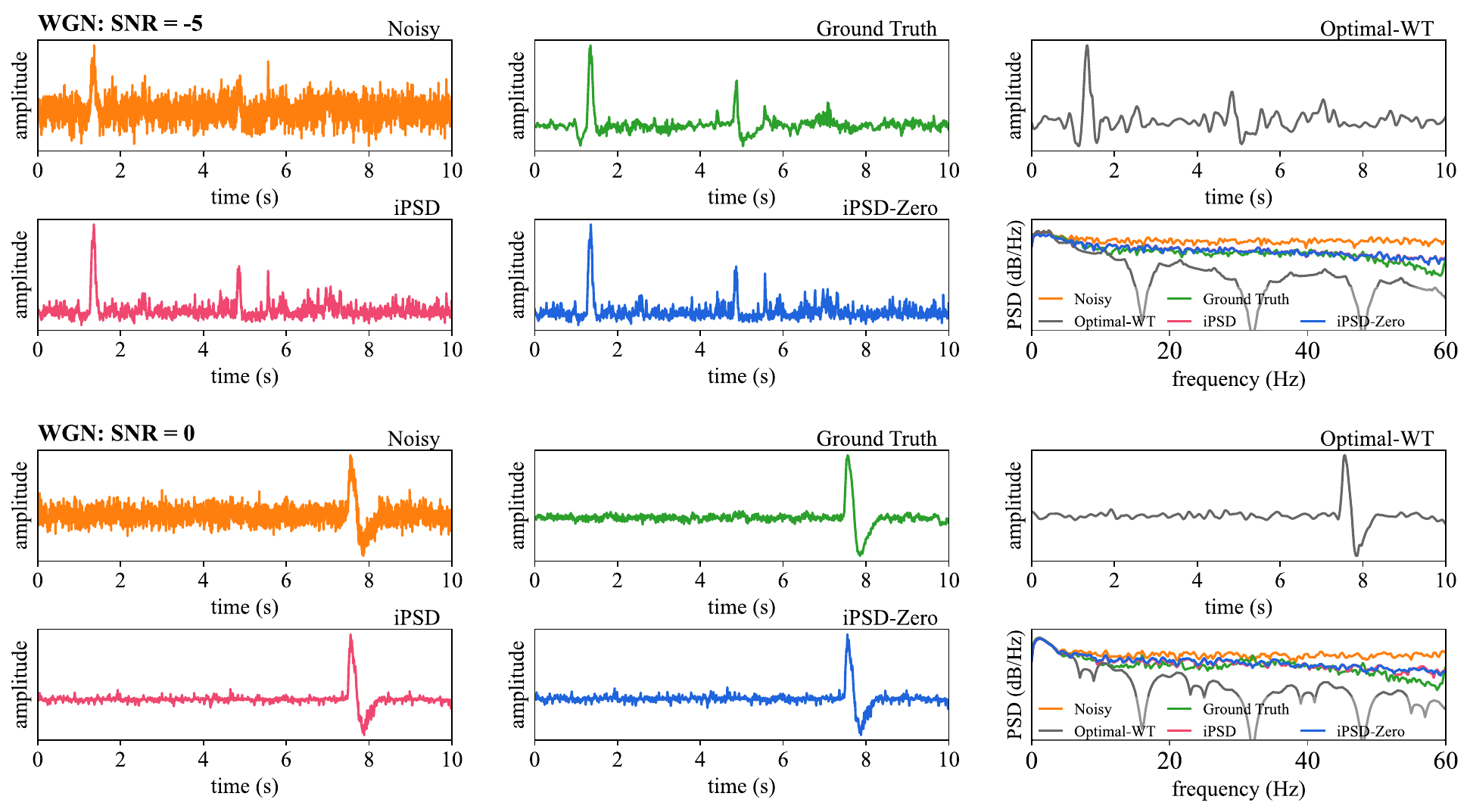}

\vspace{1em}

\includegraphics[width=\textwidth]{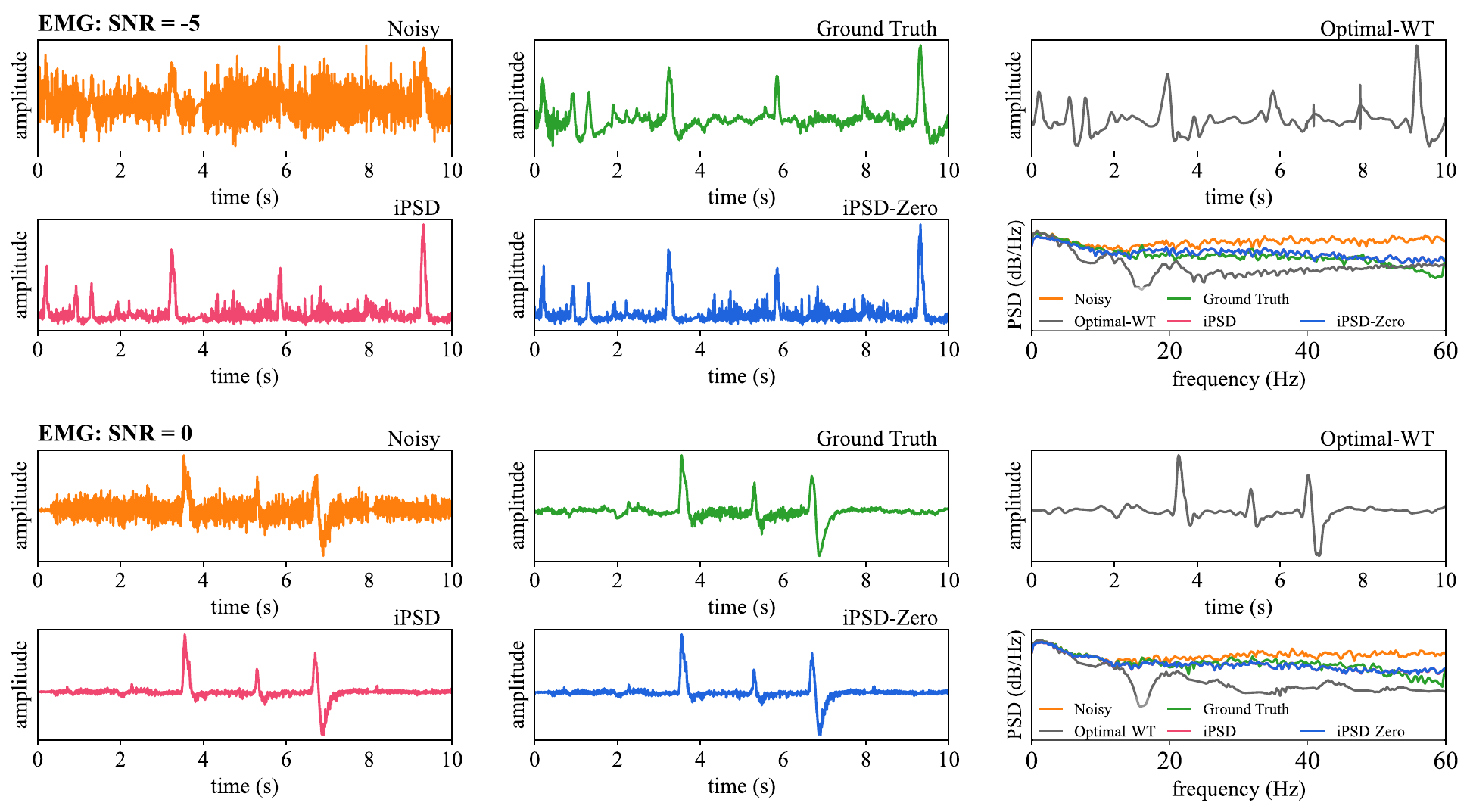}
\caption{Visualizing the denoising of example EEG segments contaminated by WGN and EMG artifacts at input SNRs of $-5$ dB and $0$ dB. For each example, the 6 panels show: the noisy input EEG (orange), the clean ground-truth EEG (green), the Optimal-WT output (gray), the iPSD output (pink), the iPSD-Zero output (blue), and the corresponding PSDs (lower-right panels).}
\label{fig:syn_perf}
\end{figure}

Under WGN contamination, although the Optimal-WT reconstruction appears visually smooth in the time domain, it incurs substantial spectral distortion, sharply attenuating components above 10 Hz and introducing spurious dips at a few frequencies, resulting in the loss of physiologically meaningful neural activity. In contrast, both iPSD and iPSD-Zero faithfully recover the underlying clean EEG, preserving the temporal morphology (e.g., the ridges and spikes) and the spectral characteristics. This highlights the advantage of iPSD's expressive DL denoiser, which captures the underlying signal structure for effective denoising, over Optimal-WT's fixed decomposition strategy, which corrupts the original EEG content.

Under EMG contamination, the advantage of iPSD becomes even more pronounced. At both input SNR levels, the iPSD and iPSD-Zero outputs closely track the ground-truth in the time domain, with EMG-induced artifacts substantially suppressed. The corresponding PSDs further confirm that iPSD selectively attenuates broadband EMG components while recovering the underlying EEG spectrum. In contrast, Optimal-WT fails to recover the clean signal, introducing spurious spikes in the time domain that are absent from the ground-truth and causing noticeable distortion in the frequency domain. Overall, iPSD-Zero matches the performance of iPSD, making it well-suited for settings that demand immediate deployment without prior training.

\section{Cross-subject generalizability}
\label{app:partition_generalization}
To investigate whether the learned partitioning strategy generalizes to unseen subjects, we conduct a leave-one-subject-out evaluation on five subjects from the CHB-MIT dataset. Each subject's recordings are corrupted with distinct EMG noise at an input SNR of $-5$ dB. In each fold, one subject is held out for evaluation, while the partitioning policy is learned using the remaining four subjects. The held-out subject and its associated noise are therefore unseen during training.
As shown in \cref{tab:subject_generalization}, iPSD achieves consistent denoising performance across all five held-out subjects, with output SNRs ranging from $3.49$ to $3.55$ dB and an average of $3.52$ dB. This indicates iPSD learns a general, transferable partitioning principle rather than overfitting to subjects or conditions.

\begin{table}[htbp]
\centering
\caption{Leave-one-subject-out evaluation on five CHB-MIT subjects
with EMG noise at an input SNR of $-5$ dB.}
\label{tab:subject_generalization}
\setlength{\tabcolsep}{5pt}
\small
\begin{tabular}{lccccc}
\toprule
Held-out subject & \#1 & \#2 & \#3 & \#4 & \#5 \\
\midrule
Output SNR (dB) & 3.55 & 3.51 & 3.49 & 3.52 & 3.53 \\
\bottomrule
\end{tabular}
\end{table}

\end{document}